\documentclass[twoside,11pt]{article}

\usepackage{blindtext}

\usepackage[utf8]{inputenc}
\usepackage[T1]{fontenc}
\usepackage{lmodern}
\usepackage{hyperref}
\usepackage{url}            
\usepackage{booktabs}       
\usepackage{amsfonts}       
\usepackage[numbers]{natbib}
\usepackage{nicefrac}       
\usepackage{algorithm,algpseudocode}
\usepackage{amsfonts}
\usepackage{multicol}
\usepackage{amsmath}
\usepackage{amsthm}
\usepackage{amssymb}
\usepackage[title]{appendix}
\usepackage{bm}
\usepackage{courier}
\usepackage[usenames,dvipsnames]{color}
\usepackage{enumitem}
\usepackage{graphicx}
\usepackage[margin=1in]{geometry}
\usepackage{url}
\usepackage{subfigure}
\usepackage{multirow}
\usepackage[dvipsnames]{xcolor}
\usepackage{epsfig}

\hypersetup{
	colorlinks,
	linkcolor={red!80!black},
	citecolor={blue!50!black},
	urlcolor={blue!80!black}
}

\title{A Provably Accurate Randomized Sampling Algorithm for Logistic Regression\footnote{\scriptsize This is an extended version of the following AAAI paper: \url{https://doi.org/10.1609/aaai.v38i10.29042}.}}

\date{}
\author{Agniva Chowdhury\footnote{\scriptsize Computer Science and Mathematics Division, Oak Ridge National Laboratory, TN, USA,\,\url{chowdhurya@ornl.gov}.}
\and
Pradeep Ramuhalli\footnote{\scriptsize Nuclear Energy and Fuel Cycle Division, Oak Ridge National Laboratory, TN, USA,\, \url{ramuhallip@ornl.gov}.
}
}

\usepackage{bibentry}

\RequirePackage{amsmath}
\RequirePackage{amsthm}
\RequirePackage{amssymb}
\newcommand{\pr}[1]{\mbox{}\mathbb{P}\left(#1\right)}

\newcommand{\ts}{\mathsf{T}}

\newcommand{\RR}[2]{\mathbb{R}^{#1 \times #2}} 
\newcommand{\R}[1]{\mathbb{R}^{#1}} 

\newcommand{\gb}{\mathbf{g}}

\newcommand{\pb}{\mathbf{p}}

\newcommand{\xb}{\mathbf{x}}
\newcommand{\yb}{\mathbf{y}}
\newcommand{\zb}{\mathbf{z}}

\newcommand{\Ib}{\mathbf{I}}

\newcommand{\Sb}{\mathbf{S}}

\newcommand{\Ub}{\mathbf{U}}
\newcommand{\Vb}{\mathbf{V}}
\newcommand{\Wb}{\mathbf{W}}
\newcommand{\Xb}{\mathbf{X}}

\newcommand{\Ocal}{\mathcal{O}}

\ifx\BlackBox\undefined
\newcommand{\BlackBox}{\rule{1.5ex}{1.5ex}}  
\fi

\ifx\QED\undefined
\def\QED{~\rule[-1pt]{5pt}{5pt}\par\medskip}
\fi

\ifx\proof\undefined
\newenvironment{proof}{\par\noindent{\bf Proof\ }}{\hfill\BlackBox\\[2mm]}
\fi

\ifx\theorem\undefined
\newtheorem{theorem}{Theorem}
\fi

\ifx\example\undefined
\newtheorem{example}{Example}
\fi

\ifx\property\undefined
\newtheorem{property}{Property}
\fi

\ifx\lemma\undefined
\newtheorem{lemma}[theorem]{Lemma}
\fi

\ifx\corollary\undefined
\newtheorem{corollary}[theorem]{Corollary}
\fi

\ifx\proposition\undefined
\newtheorem{proposition}[theorem]{Proposition}
\fi

\ifx\remark\undefined
\newtheorem{remark}[theorem]{Remark}
\fi

\ifx\corollary\undefined
\newtheorem{corollary}[theorem]{Corollary}
\fi

\ifx\definition\undefined
\newtheorem{definition}{Definition}
\fi

\ifx\conjecture\undefined
\newtheorem{conjecture}[theorem]{Conjecture}
\fi

\ifx\axiom\undefined
\newtheorem{axiom}[theorem]{Axiom}
\fi

\ifx\claim\undefined
\newtheorem{claim}[theorem]{Claim}
\fi

\ifx\assumption\undefined
\newtheorem{assumption}[theorem]{Assumption}
\fi

\newtheorem{supplemma}{Lemma}

\newcommand{\EE}{\mathbb{E}} 
\newcommand{\PP}{\mathbb{P}} 

\newcommand*{\argmin}{\mathop{\mathrm{argmin}}}
\newcommand*{\argmax}{\mathop{\mathrm{argmax}}}

\newcommand{\rank}{\mathop{\mathrm{rank}}}

\newcommand{\smallfrac}[2]{{\textstyle \frac{#1}{#2}}}

\newcommand{\eg}{\emph{e.g.}}
\newcommand{\ie}{\emph{i.e.}}

\newcommand{\nbr}[1]{\left\|#1\right\|}

\newcommand{\abs}[1]{\left|#1\right|}

\newcommand{\one}{\mathbf{1}}  
\newcommand{\zero}{\mathbf{0}} 

\newcommand{\Sigmab}{\mathbf{\Sigma}}

\newcommand{\betab}{\boldsymbol{\beta}}

\newcommand{\ve}{\varepsilon}

\newcommand{\var}{\mathrm{Var}}

\begin{document}

\maketitle

\begin{abstract}
In statistics and machine learning, logistic regression is a widely-used supervised learning technique primarily employed for binary classification tasks. When the number of observations greatly exceeds the number of predictor variables, we present a simple, randomized sampling-based algorithm for logistic regression problem that guarantees high-quality approximations to both the estimated probabilities and the overall discrepancy of the model. Our analysis builds upon two simple structural conditions that boil down to randomized matrix multiplication, a fundamental and well-understood primitive of randomized numerical linear algebra. We analyze the properties of estimated probabilities of logistic regression when leverage scores are used to sample observations, and prove that accurate approximations can be achieved with a sample whose size is much smaller than the total number of observations. To further validate our theoretical findings, we conduct comprehensive empirical evaluations. 
Overall, our work sheds light on the potential of using randomized sampling approaches to efficiently approximate the estimated probabilities in logistic regression, offering a practical and computationally efficient solution for large-scale datasets.
\end{abstract}

\section{Introduction}\label{sxn:intro}
In statistics and machine learning, logistic regression~\cite{hosmer2013applied} is a widely-used supervised learning technique applied to binary classification tasks. It is a statistical method that predicts one of two possible outcomes based on the input features. More specifically, the goal is to model the probability of one of the binary outcomes based on the predictor variables. 
In machine learning and various scientific applications, logistic regression appears in numerous settings, including online learning \cite{zhang2012efficient}, feature selection \cite{koh2007method}, anomaly detection~\cite{hendrycks2018deep, feng2014robust}, disease classification \cite{liao2007logistic, chai2018novel}, image \& signal processing \cite{dong2019single, rosario2004highly}, probability calibration \cite{kull2019beyond} and many more. 

Formally, given the data matrix $\Xb\in\RR{n}{d}$ and the binary response vector $\yb \in \{0,1\}^n$, logistic regression models the following
\begin{equation}\pr{y_i=1 \mid \Xb_{i*}}:=p_i(\boldsymbol{\beta})=\frac{\exp \left(\Xb_{i*} \boldsymbol{\beta}\right)}{1+\exp \left(\Xb_{i*} \boldsymbol{\beta}\right)}\label{eq:prob},
\end{equation}
for $i=1,2, \ldots, n$. Here, $y_i\in\{0,1\}$ is the $i$-th component of $\yb$, $\Xb_{i*}$ is $i$-th row (as a row vector) of $\Xb$,  and $\boldsymbol{\beta}\in\R{d}$ is the vector of unknown regression coefficients, which is often estimated by the maximum likelihood estimator (MLE) \ie, through maximizing the log-likelihood function with respect to $\boldsymbol{\beta}$, which is given by 
%
{
\begin{flalign}
\ell(\betab)=&~\sum_{i=1}^n \big(y_i \log p_i(\boldsymbol{\beta})+ (1-y_i)\log (1-p_i(\boldsymbol{\beta}))\big)\nonumber\\
=&~ \sum_{i=1}^n \big(y_i\Xb_{i*}\betab-\log(1+\exp(\Xb_{i*}\betab))\big)\label{eq:logistic}\,,
\end{flalign}
}
%
where we get eqn.~\eqref{eq:logistic} by using eqn.~\eqref{eq:prob} in the previous step. The MLE of the coefficients vector $\betab$ can be written as $\betab^*=\argmax _{\boldsymbol{\beta}}\ell(\betab)\,.$
Equivalently, eqn.~\eqref{eq:logistic} can be written in the following compact form: 
\begin{flalign}
\ell(\betab)=\yb^\ts\Xb\betab-\one^\ts \gb(\betab)\label{eq:compact}
\end{flalign}
where $\gb(\betab)$ is an $n\times 1$ vector with the $i$-th entry $g_i(\betab)= \log(1+\exp(\Xb_{i*}\betab))$, for $i=1,\dots, n$. The MLE $\betab^*$ satisfies the following condition
{
\begin{flalign}
\frac{\partial \ell\left(\boldsymbol{\beta}\right)}{\partial \boldsymbol{\beta}}\bigg|_{\betab=\betab^*}=\zero~~\Rightarrow\Xb^{\ts}\left(\mathbf{y}-\mathbf{p}(\betab^*)\right)=\zero\label{eq:normaleq}\,.
\end{flalign}
}
Here $\pb(\betab^*)$ is an $n$-dimensional vector of \emph{estimated probabilities}
\footnote{The term ``predicted probabilities" is used in some literature, but throughout our paper, we consistently refer to it as the vector of ``estimated probabilities."}
, with the $i$-th entry corresponds to $p_i(\betab^*)$, for $i=1,\dots n$. Unfortunately, due to the non-linearity of $\ell(\betab)$,  there is no closed-form analytical solution to eqn.~\eqref{eq:normaleq}. As a result, a variant of Newton's method, namely, \emph{iteratively reweighted least squares} (IRLS) \cite{green1984iter} is commonly used to find $\betab^*$ from eqn.~\eqref{eq:normaleq}, that maximizes $\ell(\betab)$. The IRLS algorithm iteratively computes the MLE of the parameter vector, by solving a weighted least squares problem at every iteration. Therefore, the \emph{per iteration} cost of the algorithm is dominated by the cost of solving the aforementioned weighted least squares problem at each iteration.   

In our work, we will focus on the data matrix $\Xb\in\RR{n}{d}$ with $n\gg d$ \ie, the number of observations greatly exceeds the number of predictor variables. For simplicity of exposition, we will also assume $\Xb$ is of full rank \ie, $\rank(\Xb)=d$. Now, in such $n\gg d$ setting, solving the weighted least-squares problem at each iteration of the IRLS algorithm is expensive, taking $\Ocal(nd^2)$ time, which essentially is the cost of computing the inverse of the Hessian matrix ${\small\left[\frac{\partial^2 \ell\left(\boldsymbol{\beta}\right)}{\partial \boldsymbol{\beta}\partial\betab^\ts}\right]^{-1}=-(\Xb^\ts\Wb\Xb)^{-1}}$ at each iteration, where $\Wb\in\RR{n}{n}$ is a diagonal weight matrix with the diagonal entries $p_i(\betab)(1-p_i(\betab))$, for $i=1,\dots,n$. Moreover, in many practical scenarios, obtaining labels for all $n$ observations of the response variable can be challenging, often involving expensive and lengthy experiments. Therefore, if we can only afford to obtain the responses for a small subset of the data points, a couple of natural question arises: \textit{First, if we estimate the parameter vector $\betab$ using only this limited subset of data, is it possible to use that estimate to accurately approximate the probabilities of a given class for all the $n$ instances? Second, what is the minimum sample size required to yield meaningful results?}

\subsection{Our Contribution}\label{sxn:contrib}

We introduce a randomized sampling-based algorithm for logistic regression with a novel analysis of it, which ensures accurate solutions in terms of the estimated probabilities. Our analysis relies on simple structural conditions that can be reduced to randomized matrix multiplication, a fundamental and well-understood primitive of randomized numerical linear algebra. Our main algorithm (see Algorithm~\ref{algo:main}) is analyzed in light of the following two structural conditions, which constructs a sampling-based sketching matrix $\mathbf{S} \in \mathbb{R}^{s \times n}$ (for an appropriate choice of the sketching dimension $s \ll n$), such that for any given vector $\xb\in\R{n}$ and accuracy parameter $0<\ve<1$,
%
\begin{flalign}
&~\left|\|\Ub^{\ts} \Sb^\ts\Sb\xb\|_2-\|\Ub^{\ts}\xb\|_2\right| \leq \frac{\varepsilon}{2}\|\xb\|_2,\label{eq:cond1}\\
&~\|\Ub^{\ts} \Sb^\ts\Sb(\mathbf{y}-\mathbf{p}(\betab^*))\|_2\leq \frac{\ve}{2}\left\|\mathbf{y}-\mathbf{p}(\betab^*)\right\|_2 \label{eq:cond2}
\end{flalign}
Here, $\mathbf{U} \in \mathbb{R}^{n \times d}$ contains the left singular vectors of $\mathbf{X}$.
Indeed, one can use the (exact or approximate) row leverage scores \cite{Mahoney11,mahoney2009cur} of the matrix $\mathbf{X}$ ( $c f$. Section~\ref{sxn:notations}) to satisfy the aforementioned constraints by sampling $\mathcal{O}(\nicefrac{d}{\ve^2})$ observations from $\Xb$, in which case $\mathbf{S}$ is a \emph{sampling-and-rescaling} matrix. Under these structural conditions, the output of Algorithm~\ref{algo:main} satisfies
\begin{flalign}
\|\mathbf{p}(\hat{\betab})-\mathbf{p}(\betab^*)\|_2\le \ve\,\|\yb-\mathbf{p}(\betab^*)\|_2.\label{eq:bound}
\end{flalign}

In words, our algorithm achieves an approximation bound on the estimated probabilities compared to the estimated probabilities obtained from $\betab^*$, which is the MLE based on the full data. Specifically, eqn.~\eqref{eq:bound} can be satisfied by sampling-and-rescaling $\mathcal{O}\left(\frac{d}{\varepsilon^2}\right)$ rows of $\Xb$. The bound in eqn.~\eqref{eq:bound} depends on $\|\yb-\pb(\betab^*)\|_2$ \ie, the goodness-of-fit of the full data model. It measures the overall discrepancy between the actual class labels and the probabilities assigned by the logistic regression model. A smaller value of $\|\yb-\pb(\betab^*)\|_2$ indicates a better fit of the model to the true data labels. Therefore, our bound suggests that our subsampled MLE $\hat{\betab}$ provides better approximations of the estimated probabilities when the full data model is well-suited to the data. This bound in eqn.~\eqref{eq:bound} is highly desirable as it depends on the ability of the full data model to accurately distinguish between different classes. See Section~\ref{sxn:algo} for an important remark on the tightness of our bound. Additionally, Our main result straightforwardly translates into the following relative-error bound in terms of the the overall discrepancy measure:
\begin{flalign}
\left|\|\yb-\mathbf{p}(\hat{\betab})\|_2-\|\yb-\pb(\betab^*)\|_2\right|\le\ve\|\yb-\pb(\betab^*)\|_2.
\end{flalign}
Finally, we summarize our key contributions below:
\begin{itemize}
    \item The bound in terms of the estimated probabilities constitutes one of our primary contributions (Theorem~\ref{thm:main}). In addition, it is important not only because it provides a precise approximation of the estimated probabilities, but also because it translates into a relative-error bound in terms of the overall discrepancy (Corollary~\ref{cor:mis}).
    \item Our second contribution is the obtained sampling complexity due to Algorithm~\ref{algo:main}. The sampling complexities of the relevant methods, namely, \cite{munteanu2018coresets} and \cite{mai2021coresets} are $\tilde{\Ocal}(\nicefrac{d^3 \cdot \mu_{\yb}(\Xb)^2}{\epsilon^4})$ and $\tilde{\Ocal}(\nicefrac{d \cdot \mu_{\yb}(\Xb)^2}{\epsilon^2})$ respectively, which depend on the so-called complexity measure $\mu_{\yb}(\Xb)$, quantifying the difficulty of compressing a dataset for logistic regression. The value of $\mu_{\yb}(\Xb)$ can be substantially large depending on the data. In contrast, our novel analysis eliminates the $\mu_{\yb}(\Xb)^2$ factor and our sampling complexity is $\Ocal(\nicefrac{d}{\ve^2})$, independent of $\mu_{\yb}(\Xb)$.
    \item Finally, note that \cite{munteanu2018coresets} proposed the so-called L2S method, which is a sampling scheme from a mixture distribution with one component proportional to the \emph{square root of the leverage scores}, a method that significantly differs from conventional leverage score sampling approaches. Similarly, \cite{mai2021coresets} utilized a more carefully constructed probability distribution, namely, $\ell_1$\emph{-Lewis weights}. Interpreting the impact of a data point based on the above sampling schemes might necessitate additional context and explanation, making them relatively less intuitive for practitioners. In contrast, standard leverage scores are widely used as a sampling tool due to their ease of interpretability. They provide direct, visual, and statistical insights into the importance of individual data points in the model. In this context, our third contribution is that we are the first to employ standard leverage scores sampling in logistic regression and provide strong accuracy guarantees with improved sampling complexity.  In addition, the use of leverage score sampling makes our analyses simpler and cleaner as compared to the prior works.
    %
    \end{itemize}
Additionally, we evaluate our algorithm on a variety of real datasets in order to practically assess its performance. In terms of accuracy, Algorithm~\ref{algo:main} performs comparably to both \cite{munteanu2018coresets, mai2021coresets} and the full-data model. Regarding runtime, we can easily align our method with \cite{munteanu2018coresets} by utilizing the fast computation of leverage scores from \cite{clarkson2017low}. See Section~\ref{sxn:exp} for details.

\subsection{Prior Work}
Over the past few decades, randomized numerical linear algebra has strongly advocated for the adoption of sketching and sampling techniques to compress data matrices, providing provable guarantees in various optimization problems, including linear regression \cite{DMM06}, ridge regression \cite{alaoui2015fast, CYD18} , low-rank approximation \cite{sarlos2006improved}, k-means clustering \cite{CohEldMusMusetal15}, principal components analysis \cite{BouMahDri08}, Fisher's discriminant analysis \cite{CYD19}, linear programming \cite{song2021oblivious,chowdhury2020faster,chowdhury2022faster,dexter2022convergence} and many others. Very recently, there has been a growing interest in applying such sketching techniques to logistic regression problems. In this section, we highlight our contributions in the context of this rapidly growing field of sketching-based algorithms for logistic regression.

Recent works have explored subsampling for logistic regression from statistical viewpoints, employing schemes such as a two-step subsampling approach \cite{wang2018optimal}, Poisson subsampling \cite{wang2019more}, and an information-based sampling strategy \cite{cheng2020information}. Similarly, in addressing extreme class imbalance, \cite{wang2020logistic, wang2021nonuniform} examined the randomized undersampling strategy, enhancing model performance by selecting a smaller subset from the majority class. Notably, these efforts focus on optimal strategies in asymptotic scenarios (\textit{i.e.,} $n \rightarrow \infty$) under standard assumptions. In contrast, our work assesses performance in finite data regimes. Importantly, these strategies lack finite sample guarantees and, in most cases, necessitate solving the full-data logistic regression problem for implementation.

The work more closely related to ours in terms of the accuracy bound is \cite{song2022deterministic}. If $n$ is extremely large,  \cite{song2022deterministic} proposed a hybrid sampling scheme based on both randomized and deterministic strategies, and provided non-asymptotic accuracy bounds in terms of the estimated probabilities of the logistic regression. While the deterministic scheme is based on leverage score sampling and primarily follows the two-step algorithm of \cite{wang2018optimal}, the randomized strategy relies on the sampling probabilities that implicitly depend on $\betab^*$, the full data MLE of the  model, which is not very effective in practice. Furthermore, the theoretical bounds provided by \cite{song2022deterministic} contain the condition number of the data matrix $\Xb$ in the numerator. Therefore, when $\Xb$ is ill-conditioned, the condition number can become exceedingly large, resulting in bounds that are relatively imprecise. Finally, the minimum sample size of their sampling methods is determined by $\betab^*$, which again poses practical challenges.

In another closely related line of research, as already mentioned in Section~\ref{sxn:contrib}, \cite{munteanu2018coresets, mai2021coresets} studied the so-called \emph{coresets} for logistic regression and came up with provable bounds using a smaller (and weighted) subset of the original data points of $\Xb$ sampled according to carefully constructed probability distributions, such as the so-called \emph{$\ell_1$-Lewis weights} \cite{cohen2015lp}. In particular, \cite{mai2021coresets} established the current state-of-the-art $\ve$-relative error bounds with $\tilde{\Ocal}\left(d \cdot \mu_{\yb}(\Xb)^2 / \ve^2\right)$ points, where $\mu_{\yb}(\Xb)$ is a complexity measure of the data matrix $\Xb$ and response vector $\yb \in\{-1,1\}^n$. Similarly, in more recent works, \cite{munteanu2021oblivious, munteanu2023almost} introduced data-oblivious, random projection-based sketching methods designed for logistic regression, that came with probabilistic guarantees on the sketched estimate. In $n\ll d$ regime, \cite{dexter2023feature} recently presented new bounds for coresets construction and dimensionality reduction for logistic regression problem by sketching the feature space. However. it is important to note that in all the aforementioned efforts related to the coresets of logistic regression, the theoretical bounds are defined in terms of the logistic loss function and not directly in relation to the estimated probabilities based on the MLEs of the logistic regression model. In addition, the underlying sampling complexities rely on the previously mentioned complexity measure $\mu_{\yb}(\Xb)$, a quantity that is contingent on the distribution of the data.


Finally, we refer the reader to the surveys \cite{Woodruff14,Mahoney11,DriMah16,DM2018,martinsson_tropp_2020} for more background on Randomized Numerical Linear Algebra and and its applications.

\subsection{Notations}\label{sxn:notations}
We use $\mathbf{x}, \mathbf{y}, \ldots$ to denote vectors and $\mathbf{X}, \mathbf{Y}, \ldots$ to denote matrices. For a matrix $\mathbf{X}, \mathbf{X}_{* i}\left(\mathbf{X}_{i *}\right)$ denotes the $i$-th column (row) of $\mathbf{X}$ as a column (row) vector and $\Xb_{ij}$ is the $(i,j)$-th entry of $\Xb$. For a vector $\xb$, $x_i$ is its $i$-th entry and $\|\xb\|_2$ denotes its Euclidean norm; for a matrix $\Xb$, $\|\Xb\|_2$
denotes its spectral norm and $\|\Xb\|_\text{F}$ denotes its Frobenius
norm. We refer the reader to \cite{GVL12} for properties of norms
that will be quite useful in our work.
For a matrix $\Xb \in \mathbb{R}^{n \times d}$ with $n> d$ and $\rank(\Xb)=d$, its (\emph{thin}) Singular Value Decomposition (SVD) is the product $\Ub \Sigmab \Vb^\ts$, with $\Ub \in \mathbb{R}^{n \times d}$ (the matrix of the left singular vectors), $\Vb \in \mathbb{R}^{d \times d}$ (the matrix of the right singular vectors), and $\Sigmab \in \mathbb{R}^{d \times d}$ a diagonal matrix whose diagonal entries are the non-zero singular values of $\Xb$ arranged in non-increasing order. Computation of the SVD takes, in this setting, $\Ocal(nd^2)$ time.
Finally, the row \emph{leverage scores} of~$\Xb$ are given by
$\|\Ub_{i*}\|_2^2$
for $i=1,2,\dots,n$.
Additional notation will be introduced as needed.
\section{Our Approach}\label{sxn:algo}

\begin{algorithm}[tb]
\caption{Construct $\Sb$}
\label{algo:constructS}
\textbf{Input}: Sampling probabilities $\pi_i$, $i=1,\dots, n$, number of sampled indices $s \ll n$;\\[1mm]
\textbf{Output}: Sampling-and-rescaling matrix $\Sb\in\RR{s}{n}$;\\[1mm]
\textbf{Initialize}: $\Sb \gets \zero_{s \times n}$;\\[-3mm]
\begin{algorithmic}[1] 
\For{$i=1$ \textbf{to} $s$}
\State Pick $j_i\in\left\{1,\ldots ,n\right\}$ with $\pr{j_i=k}=\pi_k$;
\State $\Sb_{ij_i} \gets (s\,\pi_{j_i})^{-\frac{1}{2}}$;
\EndFor
\State \textbf{return} $\Sb$
\end{algorithmic}
\end{algorithm}

\subsection{Constructing the Sketching Matrix \texorpdfstring{$\Sb$}{Sb}}\label{sxn:constructS}
%
%
We construct the  sampling-based sketching matrix $\Sb$ using Algorithm~\ref{algo:constructS}, which has previously appeared in several prior literature including randomized matrix multiplication \cite{Dr06}, linear regression \cite{DMM06}, and many other. In this work, we utilize it in the context of logistic regression, which is why we provide a very brief explanation suitable for general reader.

In Algorithm~\ref{algo:constructS}, we construct the sampling-and-rescaling matrix $\mathbf{S}$ by independently selecting $s$ elements ($s\ll n$), with replacement, from a set of indices $\{1,2,\dots,n\}$, based on a pre-specified probability distribution $\{\pi_1,\pi_2,\dots \pi_n\}$, where $0<\pi_i<1$ for $i=1,2,\dots,n$, and $\sum_{i=1}^n\pi_i=1$. If the $k$-th independent random trial results in the index $\ell$, we assign {\small $\mathbf{S}_{k\ell}=\nicefrac{1}{\sqrt{s\,\pi_{\ell}}}$}; otherwise, {\small $\mathbf{S}_{k\ell}=0$}, for $k=1,\dots, s$ and $\ell=1,\dots, n$.

Now, we outline three key observations about $\Sb$ constructed this way using Algorithm~\ref{algo:constructS}. 
First, note that $\mathbf{S}$ is very sparse, having only one non-zero entry per row, resulting in a total of $s$ non-zero entries. 
Second, computing $\mathbf{S}\mathbf{X}$ is equivalent to selecting $s$ rescaled rows of $\mathbf{X}$, independently and with replacement, according to the same probability distribution $\{\pi_1,\pi_2,\dots \pi_n\}$.
Third, $\Sb^\ts\Sb\in\RR{n}{n}$ is a diagonal matrix and the $\ell$-th diagonal entry of $\Sb^\ts\Sb$ is given by ${\small (\Sb^\ts\Sb)_{\ell\ell}=\frac{L}{s\,\pi_{\ell}}}$, where $L= 0,\dots, s$ denotes the number of times index $\ell$ is picked up in the sample of size $s$.
Now, we proceed to our main sampling-based algorithm for logistic regression.
\subsection{Main Algorithm}\label{sxn:main}
\begin{algorithm}[tb]
\caption{Sketched logistic regression}
\label{algo:main}
\textbf{Input}: data matrix $\Xb\in\RR{n}{d}$, response vector $\yb\in\{0,1\}^n$, sampling matrix $\Sb\in\RR{s}{n}$;\\[-2mm]

\textbf{Output}: $\hat{\betab}\in\R{d}$, $\pb(\hat{\betab})\in (0,1)^n$;\\[-2mm]
\begin{algorithmic}[1]
\State Compute $\hat{\betab} = \underset{\betab}{\argmax} \left(\yb^\ts\Sb^\ts\Sb\Xb\betab- \one^\ts\Sb^\ts\Sb\,\gb(\betab)\right);$

where $\gb(\betab)$ is defined in eqn.~\eqref{eq:compact}.

\vspace{1mm}
\State Compute $\pb(\hat{\betab})$, with $p_i(\hat{\betab})= \frac{\exp \left(\Xb_{i*} \hat{\boldsymbol{\beta}}\right)}{1+\exp \left(\Xb_{i*} \hat{\boldsymbol{\beta}}\right)}$
%

\State \textbf{return} $\hat{\betab}$, $\pb(\hat{\boldsymbol{\beta}})$;
\end{algorithmic}
\end{algorithm}
Given the sketching matrix $\Sb$ constructed using Algorithm~\ref{algo:constructS}, our main algorithm (Algorithm~\ref{algo:main}) is conceptually simple. We first modify the full data log-likelihood function in eqn.~\eqref{eq:compact} by sampling and rescaling $s$ data points, and the resulting subsampled log-likelihood can be written as
\begin{flalign}
\bar{\ell}(\betab)=\yb^\ts\Sb^\ts\Sb\Xb\betab- \one^\ts\Sb^\ts\Sb\,\gb(\betab)\,.\label{eq:subLL}
\end{flalign}
Algorithm~\ref{algo:main} then maximizes $\bar{\ell}(\betab)$ and computes the corresponding vector of estimated probabilities with respect to the maximizer $\hat{\betab}$. Since, $\Sb^\ts\Sb$ is a diagonal matrix (see Section~\ref{sxn:constructS}), we can rewrite $\bar{\ell}(\betab)$ as
$\bar{\ell}(\betab)= \sum_{i=1}^n \left(y_i(\Sb^\ts\Sb)_{ii}\Xb_{i*}\betab-(\Sb^\ts\Sb)_{ii}\,g_i(\betab)\right)$.
Recall that $g_i(\betab)= \log(1+\exp(1+\Xb_{i*}\betab))$ is the $i$-th entry of $\gb(\betab)$. Therefore, from the optimality condition $\frac{\partial \bar{\ell}\left(\boldsymbol{\beta}\right)}{\partial \boldsymbol{\beta}}|_{\betab=\hat{\betab}}=\zero$, we have
\begin{flalign}
\Xb^{\ts}\Sb^\ts\Sb\big(\mathbf{y}-\mathbf{p}(\hat{\betab})\big)=\zero\,.\label{eq:submle}
\end{flalign}

Theorem~\ref{thm:main} presents our approximation guarantee under
the assumption that the sketching matrix $\Sb$ satisfies the
constraints of eqns. (\ref{eq:cond1}) and (\ref{eq:cond2}).

\begin{theorem}\label{thm:main}
Let $\Xb\in\RR{n}{d}$ and $\yb\in\{0,1\}^n$ be the inputs of the logistic regression problem. Assume that for some constant $0<\ve< 1$, the sketching matrix $\Sb \in \mathbb{R}^{s\times n}$ satisfies the structural conditions of eqns.~(\ref{eq:cond1}) and \eqref{eq:cond2}. Then, the estimator $\hat{\betab}$ returned by Algorithm~\ref{algo:main} satisfies
\begin{flalign*}
\|\mathbf{p}(\hat{\betab})-\mathbf{p}(\betab^*)\|_2\le \ve\,\|\yb-\mathbf{p}(\betab^*)\|_2.
\end{flalign*}
%
Recall that $\mathbf{p}(\betab^*)$ is the vector of estimated probabilities from the full data MLE of the logistic regression coefficients.
\end{theorem}
%
%

Further insights are required to better understand the above bound.

\begin{remark}
As mentioned in Section~\ref{sxn:contrib}, the tightness of our bound depends on the performance of the full data model based on $\betab^*$. A smaller value of the residual $\|\yb-\pb(\betab^*)\|_2$ indicates a better fit of the full data model to the true labels and our bound becomes tighter. In fact, when there is no misclassification in the model with respect to the full data MLE $\betab^*$, our bound is even tighter than the relative-error bound. To illustrate this, if the logistic regression model with coefficient vector $\betab^*$ finds a decision boundary that perfectly separates the two classes, then for all observations with $y_i=1$, we have $(y_i-p_i(\betab^*))\le p_i(\betab^*)$, and for all $y_i=0$, we trivially have $(y_i-p_i(\betab^*))= -p_i(\betab^*)$. This implies $\|\yb-\pb(\betab^*)\|_2^2\le \|\pb(\betab^*)\|_2^2$. Therefore, our bound becomes tighter than $\|\mathbf{p}(\hat{\betab})-\mathbf{p}(\betab^*)\|_2\le \ve\|\mathbf{p}(\betab^*)\|_2$.
\end{remark}


As mentioned in Section~\ref{sxn:contrib}, Theorem~\ref{thm:main} further translates into the following relative-error bound in terms of the discrepancy measure. 
\begin{corollary}\label{cor:mis}
Let $\Xb$,$\yb$, $\Sb$ and $\ve$ are as defined in Theorem~\ref{thm:main}. Then, the estimator $\hat{\betab}$ returned by Algorithm~\ref{algo:main} satisfies
\begin{flalign*}
\left|\|\yb-\mathbf{p}(\hat{\betab})\|_2-\|\yb-\pb(\betab^*)\|_2\right|\le\ve\|\yb-\pb(\betab^*)\|_2.
\end{flalign*}
\end{corollary}
\begin{proof}
We use the reverse triangle inequality: $$\|\mathbf{p}(\hat{\betab})-\mathbf{p}(\betab^*)\|_2 \ge \left| \|\yb-\mathbf{p}(\hat{\betab})\|_2 - \|\yb-\pb(\betab^*)\|_2 \right|.$$ Applying this to Theorem~\ref{thm:main} yields the desired result.
\end{proof}
Clearly, our bound serves two purposes simultaneously: it ensures that both our estimated probabilities (Theorem~\ref{thm:main}) and the degree of misclassification (Corollary~\ref{cor:mis}) are comparable to those of the full data model.

\section{Proof of Theorem~\ref{thm:main}}\label{sxn:proofs}
In this section, we will prove Theorem~\ref{thm:main}.
In the proofs, we will use the abbreviations $\pb^*$ for $\mathbf{p}(\betab^*)$ and $\hat{\pb}$ for $\mathbf{p}(\hat{\betab})$ to simplify the notation and make it more concise.
We remind the reader that $\mathbf{U} \in$ $\mathbb{R}^{n \times d}, \mathbf{V} \in \mathbb{R}^{d \times d}$ and $\boldsymbol{\Sigma} \in \mathbb{R}^{d \times d}$ are, respectively, the matrices of the left singular vectors, right singular vectors and singular values of $\mathbf{X}$ in a thin SVD representation. 
%
Our first result provides an important identity that will be crucial in proving the final bound. 
\begin{lemma}\label{lem:id}
Prove that $$\Ub^{\ts} \Sb^\ts\Sb\big(\mathbf{y}-\mathbf{p}(\betab^*)\big)=\Ub^{\ts} \Sb^\ts\Sb\big(\mathbf{p}(\hat{\betab})-\mathbf{p}(\betab^*)\big)$$ 
\end{lemma}

\begin{proof}
%
We start with the following
 \begin{flalign}
\Xb^{\ts} \Sb^\ts\Sb\big(\mathbf{y}-\mathbf{p}^*\big)
=&~\Xb^{\ts} \Sb^\ts\Sb\big(\mathbf{y}-\mathbf{p}^*-\hat{\mathbf{p}}+\hat{\mathbf{p}}\big)\nonumber\\
=&~\Xb^{\ts}\Sb^\ts\Sb\big(\mathbf{y}-\hat{\mathbf{p}}\big)+\Xb^{\ts} \Sb^\ts\Sb\big(\hat{\mathbf{p}}-\mathbf{p}^*\big)\nonumber\\
=&~\Xb^{\ts} \Sb^\ts\Sb\big(\hat{\mathbf{p}}-\mathbf{p}^*\big)\label{eq:normaleq_int}\,,
\end{flalign}
where the last equality directly follows from the fact that $\Xb^{\ts}\Sb^\ts\Sb\big(\mathbf{y}-\hat{\mathbf{p}}\big)=\zero$ (from eqn.~\eqref{eq:submle}). Using the thin SVD of $\Xb$, we rewrite eqn.~\eqref{eq:normaleq_int} as
\begin{flalign}
\Vb\Sigmab\Ub^\ts \Sb^\ts\Sb\big(\mathbf{y}-\mathbf{p}^*\big)=\Vb\Sigmab\Ub^\ts \Sb^\ts\Sb\big(\hat{\mathbf{p}}-\mathbf{p}^*\big)\label{eq:id2}
\end{flalign}
The proof follows from pre-multiplying eqn.~\eqref{eq:id2} by $\Sigmab^{-1}\Vb^\ts$, and the fact that $\Vb^\ts\Vb=\Ib_d$.
\end{proof}


Our next result provides a critical lower bound that is instrumental in bounding $\|\pb(\hat{\betab})-\pb(\betab^*)\|_2$.

\begin{lemma}\label{lem:lb}
If the condition in eqn.~\eqref{eq:cond1} is satisfied, then $$\|\Ub^{\ts} \Sb^\ts\Sb\big(\mathbf{p}(\hat{\betab})-\mathbf{p}(\betab^*)\big)\|_2\ge (1-\nicefrac{\ve}{2})\|\mathbf{p}(\hat{\betab})-\mathbf{p}(\betab^*)\|_2.$$ 
\end{lemma}

\begin{proof}
Assuming $\hat{\pb}\ne\pb^*$ (otherwise, we have nothing to prove), we rewrite the left hand side as

\begin{equation}
{
\begin{aligned}
&~\|\Ub^{\ts} \Sb^\ts\Sb\big(\hat{\pb}-\pb^*\big)\|_2=\frac{\|\Ub^{\ts} \Sb^\ts\Sb\big(\hat{\pb}-\pb^*\big)\|_2}{\|\hat{\pb}-\pb^*\|_2} \cdot \|\hat{\pb}-\pb^*\|_2\label{eq:lb1}
\end{aligned}
}
\end{equation}
We work on the first term on the right hand side of eqn.~\eqref{eq:lb1},
\begin{equation}
{
\begin{aligned}
\frac{\|\Ub^{\ts} \Sb^\ts\Sb(\hat{\pb}-\pb^*)\|_2}{\|\hat{\pb}-\pb^*\|_2}
\ge\min_{\zb\ne\zero}\frac{\|\Ub^{\ts} \Sb^\ts\Sb\zb\|_2}{\|\zb\|_2}
=\frac{\|\Ub^{\ts} \Sb^\ts\Sb\zb^*\|_2}{\|\zb^*\|_2}\label{eq:lb3}\,,
\end{aligned}
}
\end{equation}
where $\zb^*=\argmin_{\zb\ne\zero}\frac{\|\Ub^{\ts} \Sb^\ts\Sb\zb\|_2}{\|\zb\|_2}$.
From eqn.~\eqref{eq:cond1}, we have
{
\begin{flalign}
&\|\Ub^{\ts} \Sb^\ts\Sb\zb^*\|_2\geq \|\Ub^\ts\zb^*\|_2-\nicefrac{\ve}{2}\|\zb^*\|_2\nonumber\\
\Rightarrow&\frac{\|\Ub^{\ts}\Sb^\ts\Sb\zb^*\|_2}{\|\zb^*\|_2}\ge \frac{\|\Ub^\ts\zb^*\|_2}{\|\zb^*\|_2} - \frac{\ve}{2}\ge\min_{\zb\ne\zero}\frac{\|\Ub^\ts\zb\|_2}{\|\zb\|_2} - \frac{\ve}{2}\label{eq:lb2}
\end{flalign}
}
Combining eqns.~\eqref{eq:lb1},\eqref{eq:lb3} and \eqref{eq:lb2}, we further have
{
\begin{flalign}
\|\Ub^{\ts} \Sb^\ts\Sb\big(\hat{\pb}-\pb^*\big)\|_2
\ge&\left(\min_{\zb\ne\zero}\frac{\|\Ub^\ts\zb\|_2}{\|\zb\|_2} - \frac{\ve}{2}\right)\|\hat{\pb}-\pb^*\|_2\nonumber\\
=&\left(\sigma_{\min}(\Ub^\ts)-\nicefrac{\ve}{2}\right)\,\|\hat{\pb}-\pb^*\|_2\nonumber\\
=&(1-\nicefrac{\ve}{2})\|\hat{\pb}-\pb^*\|_2\nonumber\,,
\end{flalign}
}
where the first equality follows from the definition of minimum singular value of a matrix and $\sigma_{\min}(\Ub^\ts)$ denotes the minimum singular value of $\Ub^\ts$, which is equal to one as $\Ub$ has orthonormal columns. The proof is now complete.
\end{proof}
\textbf{Proof of Theorem~\ref{thm:main}.} Combining Lemma~\ref{lem:id}, Lemma~\ref{lem:lb}, and eqn.~\eqref{eq:cond2}, we directly have
\begin{equation}
{
\begin{aligned}
&\|\hat{\pb}-\pb^*\|_2\nonumber\\
&\le\frac{\|\Ub^{\ts} \Sb^\ts\Sb(\hat{\pb}-\pb^*)\|_2}{1-\nicefrac{\ve}{2}}
= \frac{\|\Ub^{\ts} \Sb^\ts\Sb\big(\mathbf{y}-\mathbf{p}^*\big)\|_2}{1-\nicefrac{\ve}{2}}\nonumber\\
&\le\frac{\frac{\ve}{2}\left\|\mathbf{y}-\mathbf{p}^*\right\|_2}{1-\nicefrac{\ve}{2}}\le\ve\,\left\|\mathbf{y}-\mathbf{p}^*\right\|_2\,,
\end{aligned}
}
\end{equation}
where the last inequality is due to the fact that $1-\ve/2>1/2$ as $0<\ve<$ 1. This concludes the proof. 
\hfill\(\square\)

\begin{figure*}[t]
\centering
\subfigure{%
  \includegraphics[width=0.32\textwidth]{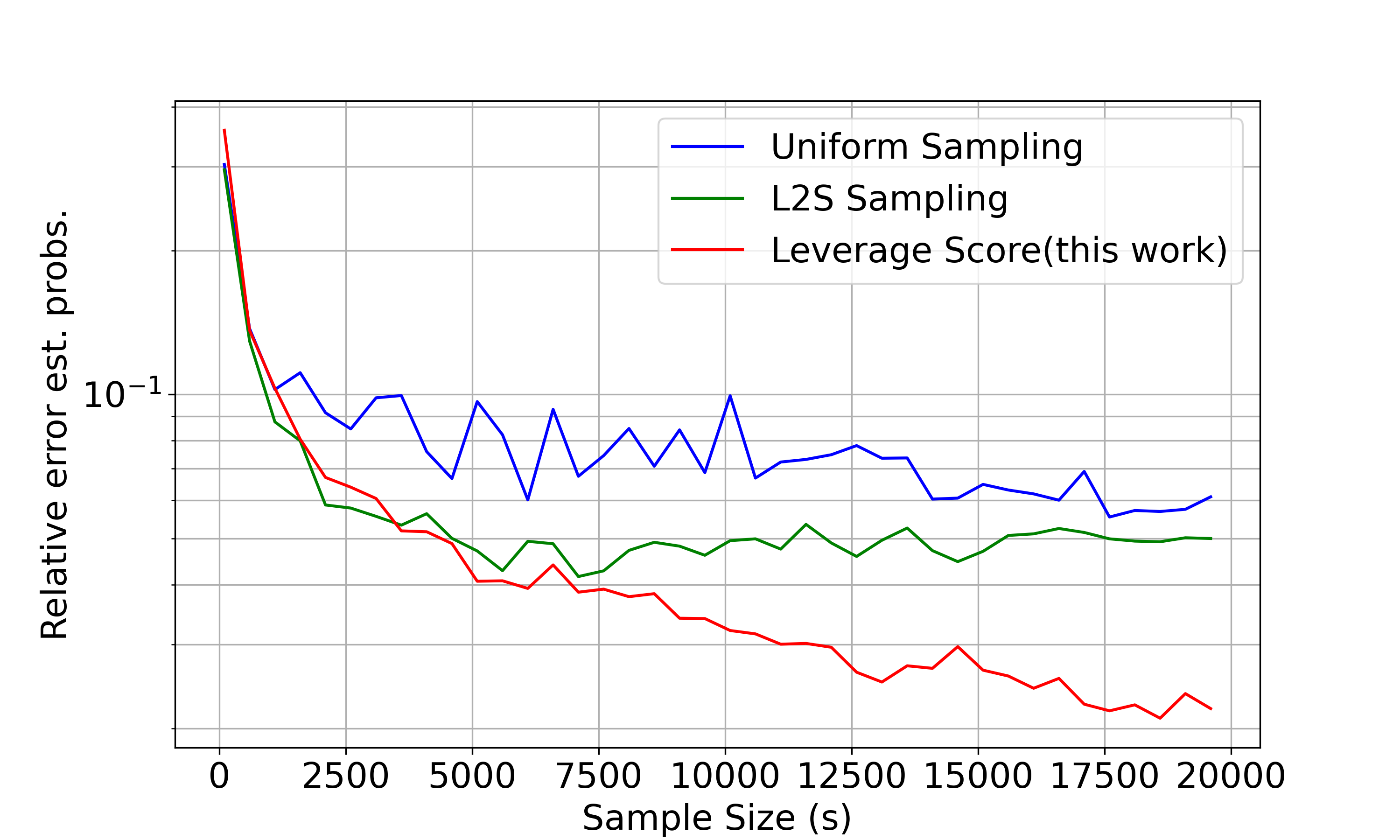}
  \label{fig:subfig10}
}
\hfill
\subfigure{%
  \includegraphics[width=0.32\textwidth]{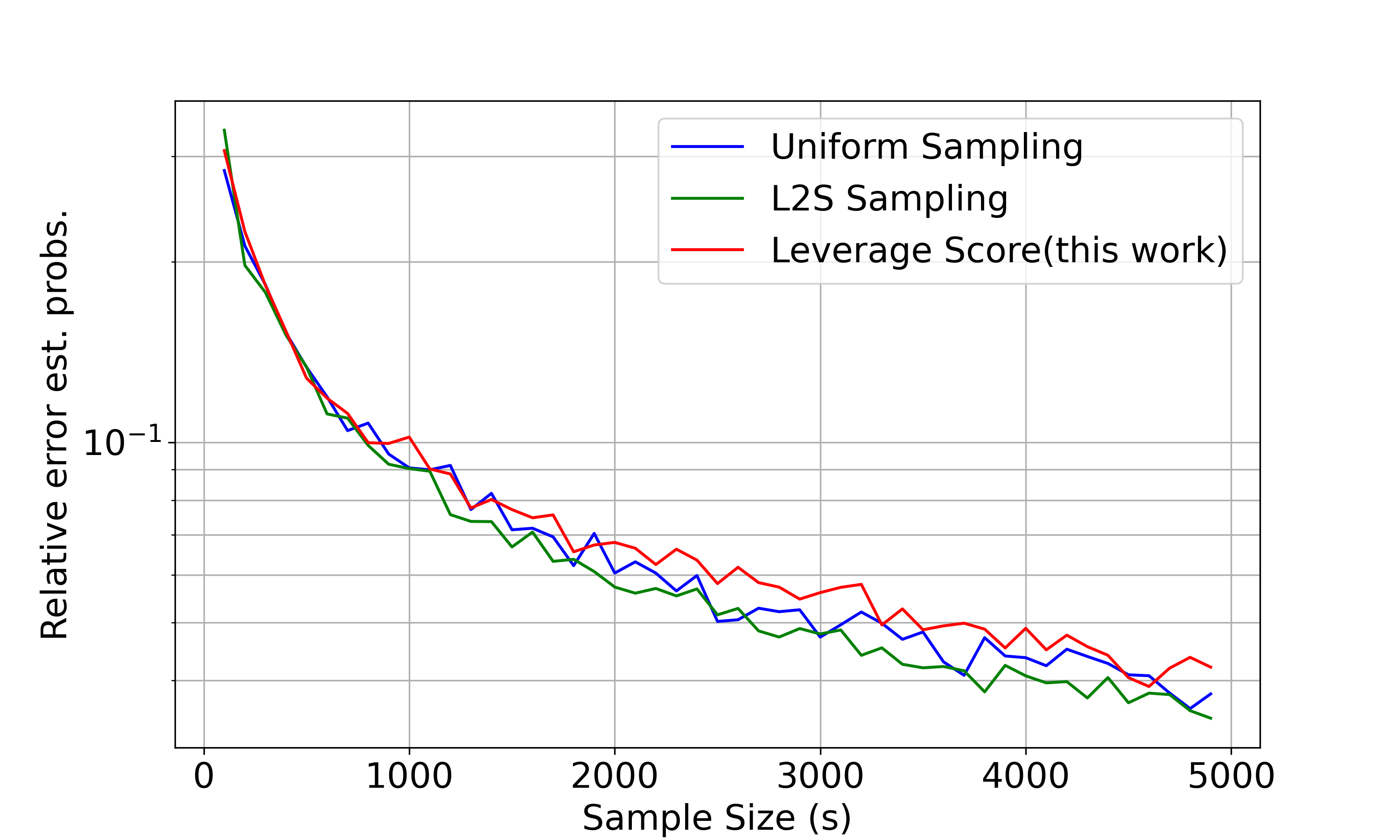}
  \label{fig:subfig11}
}
\hfill
\subfigure{
  \includegraphics[width=0.31\textwidth]{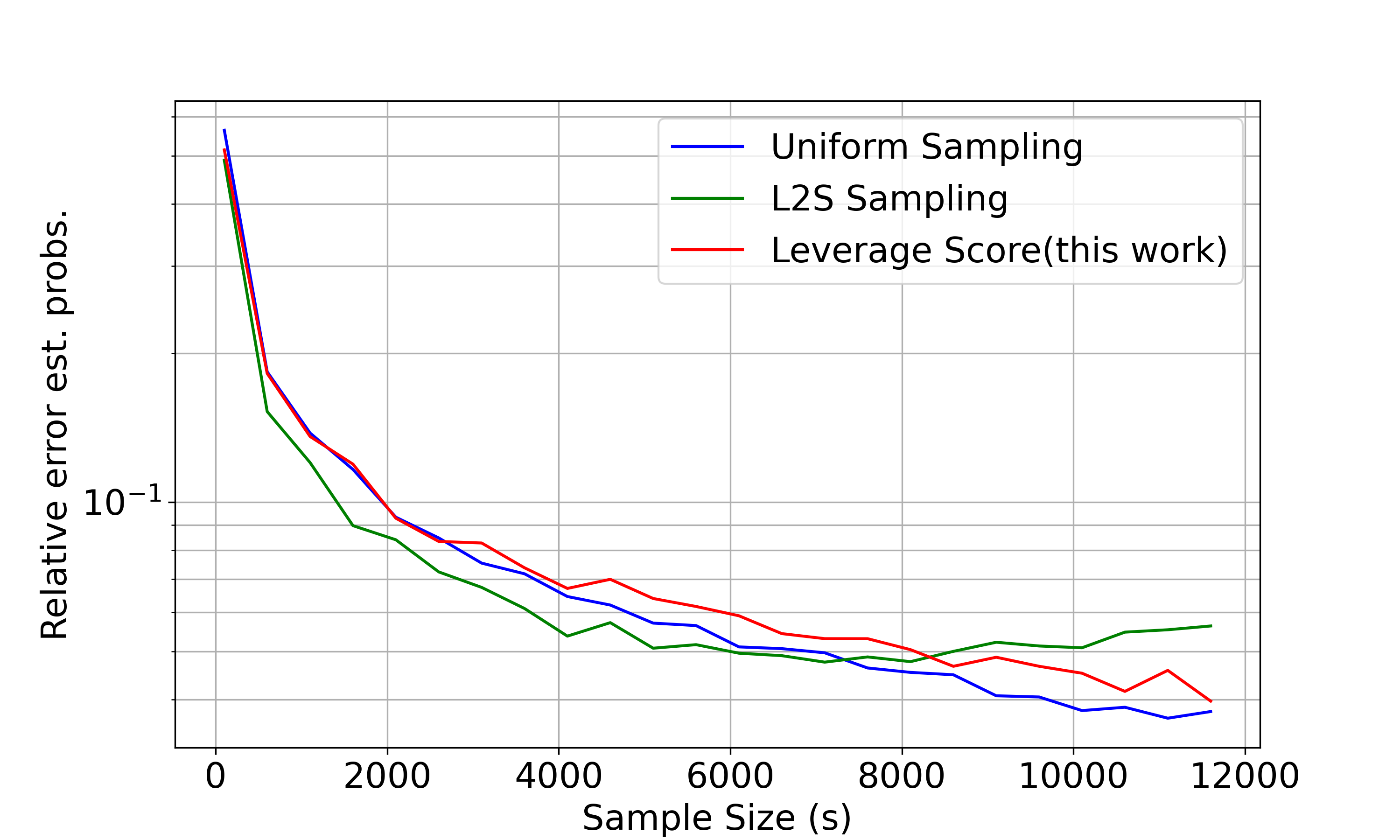}
  \label{fig:subfig12}
}\\

\addtocounter{subfigure}{-3}
\subfigure[\small cardio]{%
  \includegraphics[width=0.32\textwidth]{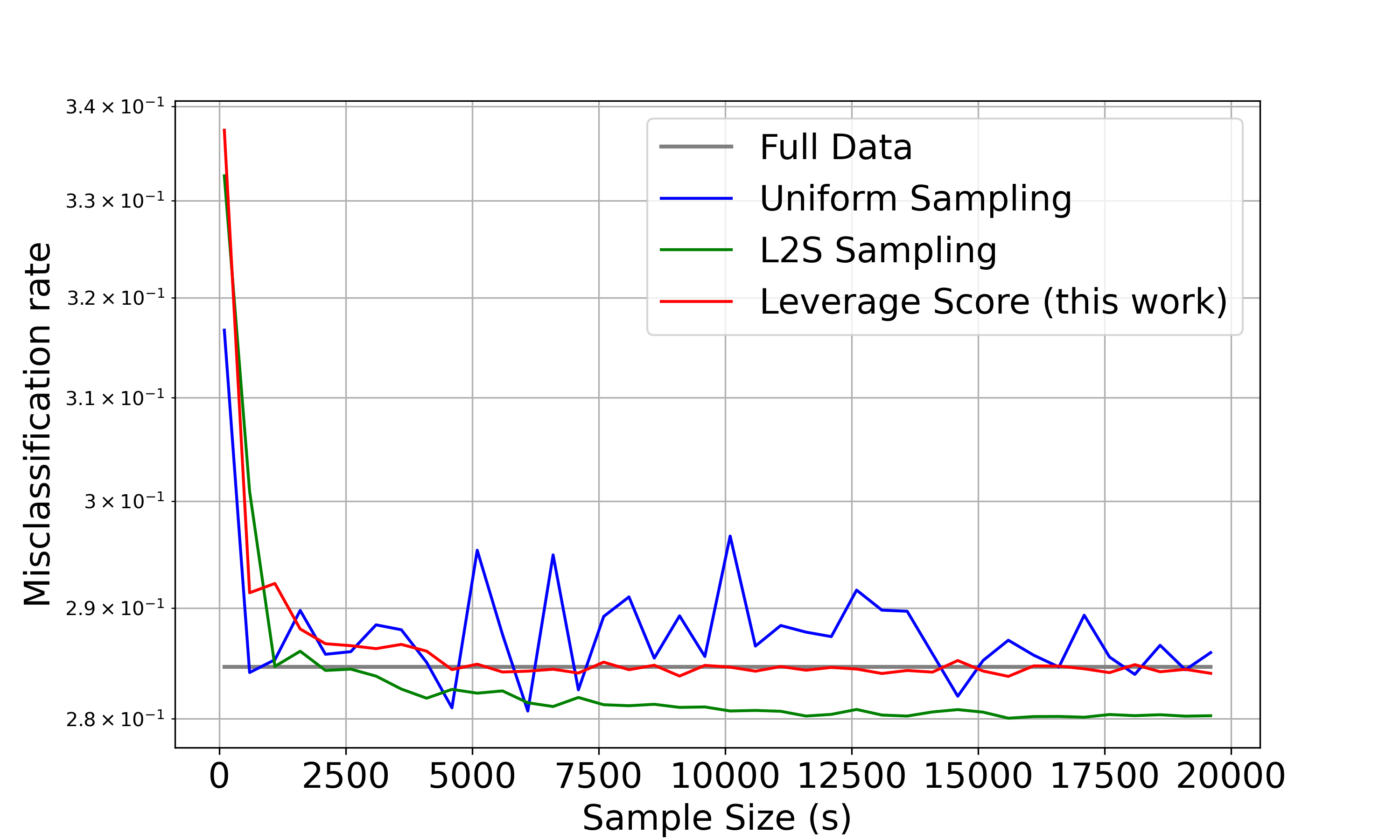}
  \label{fig:subfig13}
}
\hfill
\subfigure[\small churn]{%
  \includegraphics[width=0.32\textwidth]{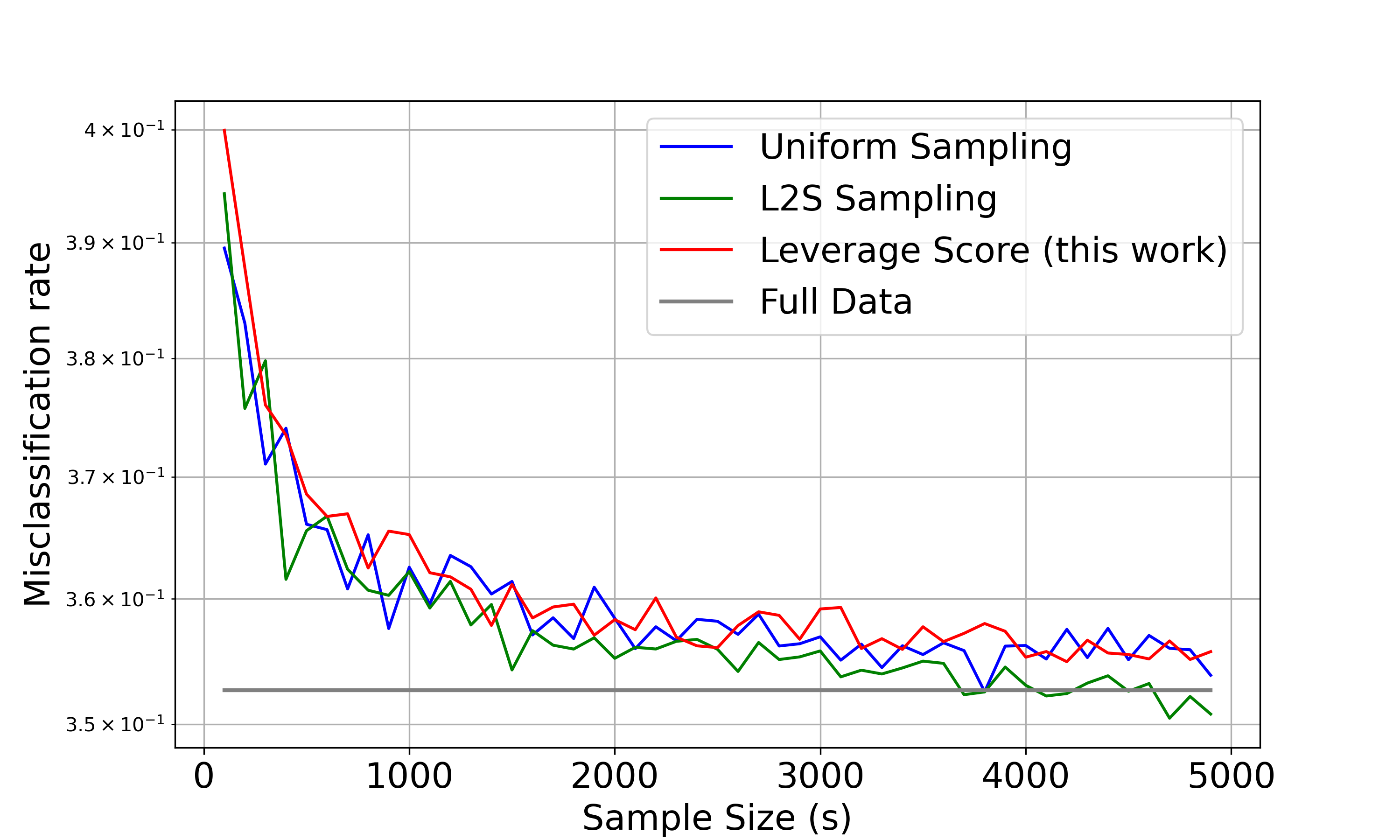}
  \label{fig:subfig14}
}
\hfill
\subfigure[\small default]{%
  \includegraphics[width=0.32\textwidth]{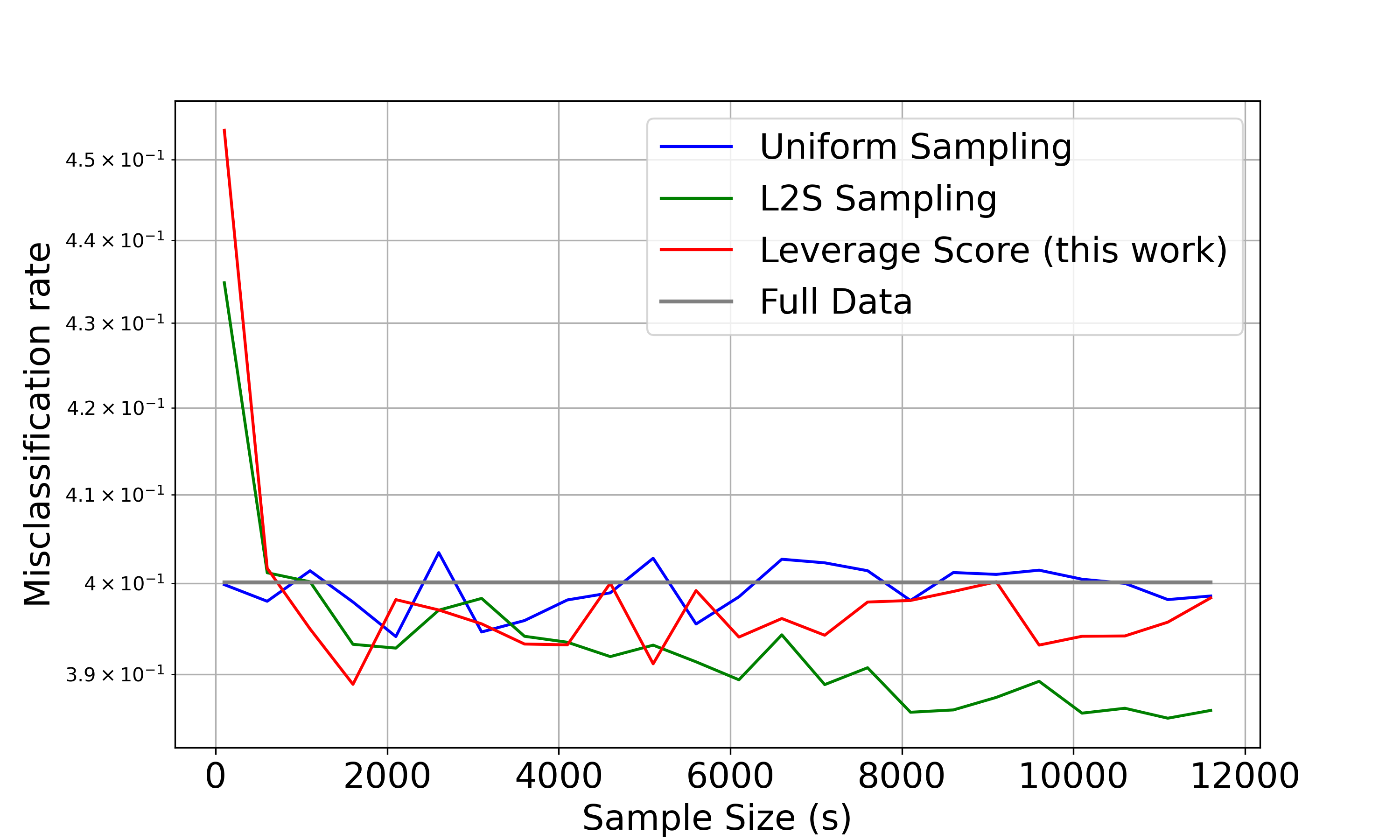}
  \label{fig:subfig15}
}\\[-4mm]

\caption{
Experiment results on real data:  
The top row of plots illustrates the relative errors in estimated probabilities and the bottom row shows misclassification rates. 
Errors are in log-scale.
}
\label{fig:mainfig2}
\end{figure*}

\section{Satisfying the Structural Conditions}
In this section, we demonstrate how to satisfy the constraints in eqns.\eqref{eq:cond1} and \eqref{eq:cond2} using the sampling-based sketching matrix $\Sb$ constructed via Algorithm~\ref{algo:constructS}. As space is limited, some of our proofs are deferred to the Appendix. Nevertheless, to offer insights into the mathematical derivations supporting our contributions, we outline the proofs as follows. Also, similar to Section~\ref{sxn:proofs}, we frequently write $\pb^*$ to denote $\pb(\betab^*)$. First, we state a fundamental result from the randomized matrix multiplication literature.
\begin{lemma}\label{lem:rmm}
Let $\mathbf{\Ub} \in \mathbb{R}^{n \times d}$ be the matrix of the left singular vectors of $\Xb$, and $\xb \in \R{n}$ be any vector. Furthermore, let $\mathbf{S} \in \mathbb{R}^{s \times n}$ is constructed using Algorithm~\ref{algo:constructS}. Then,
{
\begin{flalign}
&\mathbb{E}\left(\left\|\mathbf{\Ub}^\ts \mathbf{\Sb}^{\top} \mathbf{\Sb}\xb-\mathbf{U}^\ts\xb\right\|_2^2\right) \leq \sum_{i=1}^n \frac{\left\|\mathbf{U}_{i*}\right\|_2^2 \cdot x_i^2}{s \pi_i}\nonumber
\end{flalign}
}
\end{lemma}
Lemma~\ref{lem:rmm}, with a more general formulation, was originally introduced in \cite{Dr06} where $\Ub$ can be any matrix, the vector $\xb$ can be replaced with another matrix, and the left-hand side represents the expectation of squared Frobenius norm. However, for our specific purpose, we narrow it down to $\Ub$ being the matrix of the left singular vectors of $\Xb$, and $\xb$ being a vector. 
For completeness, we prove it in the Appendix~\ref{app:rmm}.

Next result is a special case of Lemma~\ref{lem:rmm} where the $\pi_i$'s \ie, the sampling probabilities are proportional to the row leverage scores of $\Xb$.
\begin{lemma}\label{lem:lev}
Let matrix $\mathbf{\Ub}$ and vector $\xb$ are as defined in Lemma~\ref{lem:rmm}. If the sketching matrix $\mathbf{S} \in \mathbb{R}^{s \times n}$ is constructed using Algorithm~\ref{algo:constructS}, with sampling probabilities 
$\pi_i=\nicefrac{\|\Ub_{i*}\|_2^2}{\|\Ub\|_{\mathrm{F}}^2}$,
for $i=1,\dots,n$. Then,
{
\begin{flalign}
&\mathbb{E}\left(\left\|\mathbf{\Ub}^\ts \mathbf{\Sb}^{\top} \mathbf{\Sb}\xb-\mathbf{U}^\ts\xb\right\|_2^2\right) \leq \frac{d}{s}\,\|\xb\|_2^2\nonumber
\end{flalign}
}
\end{lemma}
Proof of Lemma~\ref{lem:lev} is immediate from Lemma~\ref{lem:rmm} with $\pi_i=\smallfrac{\|\Ub_{i*}\|_2^2}{\|\Ub\|_{\mathrm{F}}^2}$, and the fact that $\|\Ub\|_{\mathrm{F}}^2=d$ (because $\Ub^\ts\Ub=\Ib_d$).
\subsection{Sample Complexity}
For the condition in eqn.~\eqref{eq:cond2}, we apply the Markov's inequality with Lemma~\ref{lem:lev} and use the fact that $\Ub^{\ts} (\mathbf{y}-\mathbf{p}(\betab^*))=\zero$ (this can be directly derived by applying thin SVD of $\Xb$ on eqn.~\eqref{eq:normaleq} and pre-multiplying the resulting equation by $\Sigmab^{-1}\Vb^\ts$),
\begin{flalign}
&\PP\left(\nbr{\Ub^{\ts}\Sb^\ts\Sb(\mathbf{y}-\mathbf{p}^*)}_2\ge\frac{\ve}{2}\nbr{\mathbf{y}-\mathbf{p}^*}_2\right)\nonumber\\
=&\PP\left(\nbr{\Ub^{\ts}\Sb^\ts\Sb(\mathbf{y}-\mathbf{p}^*)-\Ub^{\ts} (\mathbf{y}-\mathbf{p}^*)}_2\ge\frac{\ve}{2}\nbr{\mathbf{y}-\mathbf{p}^*}_2\right)\nonumber\\
\le&\frac{4\,\EE\left(\nbr{\Ub^{\ts}\Sb^\ts\Sb(\mathbf{y}-\mathbf{p}^*)-\Ub^{\ts} (\mathbf{y}-\mathbf{p}^*)}_2^2\right)}{\ve^2\nbr{\mathbf{y}-\mathbf{p}^*}_2^2}\nonumber\\
\le&\frac{4\,d\,\nbr{\mathbf{y}-\mathbf{p}^*}_2^2}{s\ve^2\,\nbr{\mathbf{y}-\mathbf{p}^*}_2^2}=\frac{4\,d}{s\ve^2}\label{eq:cond2proof}\,.
\end{flalign}

For the condition in eqn~\eqref{eq:cond1}, we have from the reverse triangle inequality, 
\begin{flalign}
&\abs{\nbr{\Ub^{\ts}\Sb^\ts\Sb\xb}_2-\nbr{\Ub^{\ts}\xb}_2}\le \nbr{\Ub^{\ts}\Sb^\ts\Sb\xb-\Ub^{\ts}\xb}\nonumber\\
\ie,&\left(\abs{\|\Ub^\ts\Sb^\ts\Sb\xb\|_2-\|\Ub^\ts\xb\|_2\|}\ge \frac{\ve}{2}\|\xb\|_2\right)~~\text{implies}~\left(\|\Ub^\ts\Sb^\ts\Sb\xb-\Ub^\ts\xb\|_2\ge \frac{\ve}{2}\|\xb\|_2\right)\,,\nonumber\\
\ie,&~\PP\left(\abs{\|\Ub^\ts\Sb^\ts\Sb\xb\|_2-\|\Ub^\ts\xb\|_2\|}\ge \frac{\ve}{2}\|\xb\|_2\right)~\le\PP\left(\|\Ub^\ts\Sb^\ts\Sb\xb-\Ub^\ts\xb\|_2\ge \frac{\ve}{2}\|\xb\|_2\right)\label{eq:cond1id}
\end{flalign}
Similar to eqn.~\eqref{eq:cond2proof}, applying Markov's inequality on the right hand side of eqn.~\eqref{eq:cond1id} and using Lemma~\ref{lem:lev}, we have
\begin{flalign}
\PP\left(\abs{\|\Ub^\ts\Sb^\ts\Sb\xb\|_2-\|\Ub^\ts\xb\|_2\|}\ge \frac{\ve}{2}\|\xb\|_2\right)\le \frac{4\,d}{s\ve^2}\label{eq:cond1proof}\,.
\end{flalign}

Now, for a failure probability $0<\delta<1$, if we set the sample size $s\ge\frac{8d}{\delta\ve^2}$, eqns.~\eqref{eq:cond2proof} and\eqref{eq:cond1proof} boil down to
\begin{flalign}
&~\PP\left(\nbr{\Ub^{\ts}\Sb^\ts\Sb(\mathbf{y}-\mathbf{p}^*)}_2\ge\frac{\ve}{2}\nbr{\mathbf{y}-\mathbf{p}^*}_2\right)\le\frac{\delta}{2}\label{eq:union1}\\
&~\PP\left(\abs{\|\Ub^\ts\Sb^\ts\Sb\xb\|_2-\|\Ub^\ts\xb\|_2\|}\ge \frac{\ve}{2}\|\xb\|_2\right)\le \frac{\delta}{2}\label{eq:union2}
\end{flalign}

Finally, applying the union bound to eqns.~\eqref{eq:union1} and \eqref{eq:union2}, we conclude that
if the number of sampled rows $s$ satisfies
$$s\ge\frac{8d}{\delta\ve^2}\,,$$
then both structural conditions of Theorem~\ref{thm:main}, namely  eqns.~\eqref{eq:cond1} and \eqref{eq:cond2} hold with probability at least $1-\delta$.

\begin{remark}
Here, it's worth highlighting that only a single structural condition, namely $\|\Ub^{\ts} \Sb^\ts\Sb\xb-\Ub^\ts\xb\|_2\leq \ve\,\left\|\xb\right\|_2$, suffices to establish our bound. Indeed, employing the lower bound of the triangle inequality to the aforementioned constraint leads to the condition presented in eqn.~\eqref{eq:cond1}. Similarly, the condition described in eqn.~\eqref{eq:cond2} can be readily deduced by simply taking $\xb=\yb-\pb^*$ and utilizing the fact that $\Ub^\ts(\yb-\pb^*)=\mathbf{0}$ in this context. This way, we can further reduce the constant in the lower bound on sample size $s$ mentioned above. However, for the sake of clarity for general readers, we break this down into two separate conditions: one for the upper bound and one for the lower bound on $\|\Ub^{\ts} \Sb^\ts\Sb\big(\hat{\pb}-\pb^*\big)\|_2$.    
\end{remark}


\noindent
\textbf{Running Time.}
As discussed in Section~\ref{sxn:intro}, 
cost of computing the full data MLE $\betab^*$ is dominated by the cost of computing the inverse of the Hessian matrix, $-(\Xb^\top\Wb\Xb)^{-1}$, at each iteration of the IRLS algorithm, which takes $\Ocal(nd^2)$ time.
In contrast, our proposed Algorithm~\ref{algo:main} offers a more efficient approach. In our setting, the inverse of the Hessian matrix of the subsampled log-likelihood function $\bar{\ell}(\betab)$ is given by,
$
{\small\left[\frac{\partial^2 \bar{\ell}(\boldsymbol{\beta})}{\partial \boldsymbol{\beta}\partial\betab^\top}\right]^{-1} = -(\Xb^\top\Wb^{1/2}\Sb^\top\Sb\Wb^{1/2}\Xb)^{-1}}.
$
Recall that $\Wb\in\RR{n}{n}$ is a diagonal matrix. Therefore, with our chosen sample size $s={\small\mathcal{O}(\nicefrac{d}{\epsilon^2}})$, this can be computed in $\mathcal{O}(\text{nnz}(\Xb) + \nicefrac{d^3}{\epsilon^2})$ time, where $\text{nnz}(\Xb)$ represents the number of non-zero elements in matrix $\Xb$. Additionally, it's also worth mentioning that approximate leverage scores are sufficient for our purpose. Their computation can be efficiently done without the need to compute $\Ub$, achieving a time complexity of $\Ocal\left(\mathrm{nnz}(\Xb) \log n + d^3 \log^2 d + d^2 \log n\right)$ due to \cite{clarkson2017low}.

\section{Empirical Evaluation}\label{sxn:exp}
\textbf{Datasets.} 
First, we provide a brief introduction to the datasets used in our empirical evaluations. We have applied our algorithm to three distinct real-world datasets. The first dataset, sourced from Kaggle, is the \emph{Cardiovascular disease dataset}~\cite{cardio}, featuring $70,000\times 12$ patient records with a $50\%$ positive case occurrence. This dataset aims to predict the presence of cardiovascular disease. The second dataset, also from Kaggle, is the \emph{Bank customer churn prediction dataset}, containing $10,000\times 10$ records with a $20\%$ positive case prevalence, focusing on the classification of customer departure likelihood. The third and final dataset, named the \emph{Default of credit card clients dataset}, is sourced from the UCI ML Repository \cite{misc_default_of_credit_card_clients_350}. It consists of $30,000\times 24$ records with a $22\%$ positive case ratio and aims to predict the probability of credit card default in the future.

\textbf{Comparisons and metrics.} In our experiments, we compare three different sampling schemes: selecting rows ($i$) uniformly at random, ($ii$) proportional to their row leverage scores (this work), and ($iii$) using the L2S method of \cite{munteanu2018coresets}. For each sampling method, we run Algorithm~\ref{algo:main} with varying sample sizes and measure two key metrics which are the most relevant to our analysis, namely,  (\emph{i}) the relative error of the estimated probabilities \ie $\nicefrac{\|\hat{\pb}-\pb^*\|_2}{\|\pb^*\|_2}$ and (\emph{ii}) the misclassification rates. Each experiment is run $20$ times and we report the means of the aforementioned metrics. Notably, we exclude the method from \cite{mai2021coresets} for comparison, as they already extensively compared their work with \cite{munteanu2018coresets}. The performance of \cite{mai2021coresets} closely aligns with \cite{munteanu2018coresets}, with very marginal variations observed for logistic regression. Therefore, we include only \cite{munteanu2018coresets} as we do not anticipate any significant differences in performance between ours and \cite{mai2021coresets}.

\textbf{Results.} 
The first set of results are presented in Figure~\ref{fig:mainfig2}. In the top row, we present relative errors in terms of estimated probabilities. For the first dataset (cardiovascular disease), our sampling approach based on row leverage scores (red) consistently outperforms both L2S (green) and uniform sampling (blue) and our method gets better as the sample size $s$ increases. For the remaining two datasets (last two columns of Figure~\ref{fig:mainfig2}), both the L2S and uniform sampling methods demonstrate marginally better performance in general, except for the fact that our method gets better than L2S for larger $s$ in the third column. However, it is noteworthy that the errors in all three methods are consistently very close to each other and become smaller as $s$ increases. Therefore, the crucial point to note is that our leverage score-based approach indeed works well in practice and demonstrates very comparable results to the other two methods, thereby validating our theoretical bound.
In the bottom row of Figure~\ref{fig:mainfig2}, we present a comparison of the misclassification rates. For the first and third datasets with moderate sample sizes, our leverage score-based approach achieves misclassification rates that are either lower than or very close to the misclassification rate of the full-data model (gray). In contrast, L2S performs slightly better, while uniform sampling performs slightly worse than ours. As for the second dataset (middle column), all three approaches perform comparably, with their respective misclassification rates decreasing and converging to that of the full-data model as $s$ increases. Overall, we would like to emphasize that our plots are on log scale and if we look at actual numbers on the $y$-axis, the difference with \cite{munteanu2018coresets} is indeed very small.

For completeness, we also compare our method with respect to the relative-error nagative log-likelihoods and due to space constraints, the plots are given in the Appendix\footnote{\url{https://github.com/AgnivaC/SubsampledLogisticRegression}}. For the same reason, we also postpone some additional experiments in Appendix (\eg, the plots for the standard deviations from the $20$ runs for each of the experiments conducted.

Finally, we want to highlight that our experiments are preliminary proof-of-concept showing that our leverage score-based sampling scheme for logistic regression works well in practice and performs very comparably to the prior work. While we use the \texttt{numpy.linalg.svd} routine to compute our leverage score-based sampling probabilities, \cite{munteanu2018coresets} employed a fast, randomized sketching-based implementation to compute the leverage scores which were subsequently used to calculate the L2S sampling probabilities. This method can also be seamlessly applied in our context for the leverage score computation. Therefore, given the architecture and specific optimization method, running time of our algorithm will be highly comparable to that of \cite{munteanu2018coresets}.
\section{Conclusion and Future Work}
We have presented a simple structural result to analyze a randomized sampling-based algorithm for the logistic regression problem, guaranteeing highly accurate solutions in terms of both the estimated probabilities and overall discrepancy. There are several immediate future directions that can be explored further.   
In terms of future research, it is important to explore whether similar bounds can be derived using random projection-based oblivious sketching matrices. This includes exploring techniques like sparse subspace embeddings as presented in \cite{cohen2016nearly}, very sparse subspace embeddings of \cite{clarkson2017low}, Gaussian sketching matrices, or even combinations of both approaches as outlined in \cite{Cohen2016}. The key challenge is when $\Sb$ is a general sketching matrix, $\Sb^\ts\Sb$ is not necessarily a diagonal matrix. Therefore it is not clear how to formulate an identity similar to eqn.~\eqref{eq:submle} with general sketching matrices.
Secondly, it's worth noting that we have employed the IRLS method in our algorithm as a black-box. Therefore, an obvious future direction would be to further investigate how the errors stemming from the IRLS solver propagate and affect our bound.
Lastly, logistic regression finds numerous applications in high-dimensional data scenarios including genomics and bioinformatics, medical diagnostics, image analysis and many more. Hence, exploring similar bounds in high dimensions, i.e., when $n\ll d$, would be intriguing.


\paragraph{Acknowledgments.}
We would like to thank the anonymous reviewers for their helpful comments. Special thanks to Petros Drineas for the insightful discussions. This work was supported by the Artificial Intelligence Initiative as part of the Laboratory Directed Research and Development (LDRD) Program at the Oak Ridge National Laboratory (ORNL). ORNL is operated by UT-Battelle, LLC., for the U.S. Department of Energy under Contract DE-AC05-00OR22725. This manuscript has been
authored by UT-Battelle, LLC. The US government retains and the publisher, by accepting the article
for publication, acknowledges that the US government retains a nonexclusive, paid-up, irrevocable,worldwide license to publish or reproduce the published form of this manuscript, or allow others to do so, for US government purposes. DOE will provide public access to these results of federally sponsored research in accordance with the DOE Public Access Plan (\url{http://energy.gov/downloads/doe-public-access-plan}).



\setlength{\bibsep}{6pt}
\bibliographystyle{abbrvnat}
{
\bibliography{bibliography.bib}
}

\appendix

\section{Proof of Lemma~\ref{lem:rmm}}\label{app:rmm}

\begin{supplemma}\label{lem:exp_var}
Let matrix $\mathbf{\Ub}$ and vector $\xb$ are as defined in Lemma~\ref{lem:rmm}. If the sketching matrix $\mathbf{S} \in \mathbb{R}^{s \times n}$ is constructed using Algorithm~\ref{algo:constructS}. Then, for $i=1,\dots,d$, we have
\begin{flalign}
&\text(a)~~\EE \left( ( \mathbf{U}^\ts \mathbf{S}^\ts \mathbf{S} \mathbf{x} )_i \right)= 
( \mathbf{U}^\ts\mathbf{x} )_i\\
&\text(b)~~\var \left( ( \mathbf{U}^\ts \mathbf{S}^\ts \mathbf{S} \mathbf{x} )_i \right)=
\frac{1}{s} \sum_{j=1}^n\frac{ (\mathbf{U}^\ts)_{ij}^2 x_{j}^2}{\pi_{j}}-\frac{1}{s} (\Ub^\ts\xb)_{i}^2\nonumber
%
\end{flalign}
\end{supplemma}
%
%
\begin{proof}[Proof of Part (a)] From the definition of expectation, we have,
\begin{flalign}
    \EE \left( ( \mathbf{U}^\ts \mathbf{S}^\ts \mathbf{S} \mathbf{x} )_i \right)
    =&~\EE \left( (\mathbf{U}^\ts)_{i*} \mathbf{S}^\ts \mathbf{S} \mathbf{x} \right)
    = \EE \left(\sum_{t=1}^s\frac{ (\mathbf{U}^\ts)_{ij_t} x_{j_t}}{s \pi_{j_t}} \right)
    =\sum_{t=1}^s\EE\left(\frac{ (\mathbf{U}^\ts)_{ij_t} x_{j_t}}{s \pi_{j_t}}\right)\nonumber\\
    =&~\sum_{t=1}^s \sum_{j=1}^n \frac{ (\mathbf{U}^\ts)_{ij} x_{j}}{s \pi_{j}}\,\pi_j
    =\frac{1}{s} \sum_{t=1}^s \sum_{j=1}^n  (\mathbf{U}^\ts)_{ij} x_{j}= \sum_{j=1}^n  (\mathbf{U}^\ts)_{ij}x_j
    = ( \mathbf{U}^\ts\mathbf{x} )_i\label{eq:exp}
\end{flalign}
\end{proof}
%
%
\begin{proof}[Proof of Part (b)] Since the samples are drawn independently (with replacement), from the \emph{additivity of variances} property we have,  
\begin{flalign}
\var \left( ( \mathbf{U}^\ts \mathbf{S}^\ts \mathbf{S} \mathbf{x} )_i \right)
=&~\var \left( (\mathbf{U}^\ts)_{i*} \mathbf{S}^\ts \mathbf{S} \mathbf{x} \right)
= \var \left(\sum_{t=1}^s\frac{ (\mathbf{U}^\ts)_{ij_t} x_{j_t}}{s \pi_{j_t}} \right)
=\sum_{t=1}^s\var\left(\frac{ (\mathbf{U}^\ts)_{ij_t} x_{j_t}}{s \pi_{j_t}}\right)\nonumber\\
=&~\sum_{t=1}^s\left(\EE\left(\frac{ (\mathbf{U}^\ts)_{ij_t} x_{j_t}}{s \pi_{j_t}}\right)^2-\EE^2\left(\frac{ (\mathbf{U}^\ts)_{ij_t} x_{j_t}}{s \pi_{j_t}}\right)\right)\nonumber\\
=&~\sum_{t=1}^s\left(\sum_{j=1}^n\left(\frac{ (\mathbf{U}^\ts)_{ij} x_{j}}{s \pi_{j}}\right)^2\cdot\pi_j - \left(\sum_{j=1}^n\frac{ (\mathbf{U}^\ts)_{ij} x_{j}}{s \pi_{j}}\cdot\pi_j\right)^2\right)\nonumber\\
=&~\frac{1}{s} \sum_{j=1}^n\frac{ (\mathbf{U}^\ts)_{ij}^2 x_{j}^2}{\pi_{j}}-\frac{1}{s} (\Ub^\ts\xb)_{i}^2\label{eq:var}
\end{flalign}
\end{proof}

\textbf{Proof of Lemma~\ref{lem:rmm}.} Now we are ready to prove Lemma~\ref{lem:rmm}. First, from Lemma~\ref{lem:exp_var}, we already know that $ ( \mathbf{U}^\ts \mathbf{S}^\ts \mathbf{S} \mathbf{x} )_i$ is an unbiased estimator of $( \mathbf{U}^\ts\mathbf{x} )_i$, for all $i=1,\dots d$,\,\ie, $\EE \left( ( \mathbf{U}^\ts \mathbf{S}^\ts \mathbf{S} \mathbf{x} )_i \right)= 
( \mathbf{U}^\ts\mathbf{x} )_i$\,. Using this fact, we rewrite the left hand side as
\begin{flalign}
\mathbb{E}\left[\left\|\mathbf{\Ub}^\ts \mathbf{\Sb}^{\top} \mathbf{\Sb}\xb-\mathbf{U}^\ts\xb\right\|_2^2\right]
=&~\mathbb{E}\left[\sum_{i=1}^d(\mathbf{\Ub}^\ts \mathbf{\Sb}^{\top} \mathbf{\Sb}\xb-\mathbf{U}^\ts\xb)_i^2\right]
=~\sum_{i=1}^d \EE\left[(\mathbf{\Ub}^\ts \mathbf{\Sb}^{\top} \mathbf{\Sb}\xb-\mathbf{U}^\ts\xb)_i\right]^2\nonumber\\
=&~\sum_{i=1}^d \EE\left[(\mathbf{\Ub}^\ts \mathbf{\Sb}^{\top} \mathbf{\Sb}\xb)_i-(\mathbf{U}^\ts\xb)_i\right]^2
=~\sum_{i=1}^d \EE\left[(\mathbf{\Ub}^\ts \mathbf{\Sb}^{\top} \mathbf{\Sb}\xb)_i-\EE  ( \mathbf{U}^\ts \mathbf{S}^\ts \mathbf{S} \mathbf{x} )_i\right]^2\nonumber\\
=&~\sum_{i=1}^d \var \left( ( \mathbf{U}^\ts \mathbf{S}^\ts \mathbf{S} \mathbf{x} )_i \right)\label{eqn:varint}
\end{flalign}

Now, combining \emph{Part (b)} of Lemma~\ref{lem:exp_var} and eqn.~\eqref{eqn:varint}, we further have
\begin{flalign}
\mathbb{E}\left[\left\|\mathbf{\Ub}^\ts \mathbf{\Sb}^{\top} \mathbf{\Sb}\xb-\mathbf{U}^\ts\xb\right\|_2^2\right]
=~&\sum_{i=1}^d  \sum_{j=1}^n\left[\frac{1}{s}\frac{ (\mathbf{U}^\ts)_{ij}^2 x_{j}^2}{\pi_{j}}-\frac{1}{s}(\Ub^\ts\xb)_{i}^2\right]
=~\sum_{j=1}^n\left[\frac{1}{s}\frac{ x_{j}^2\sum_{i=1}^d(\mathbf{U}^\ts)_{ij}^2 }{\pi_{j}}-\frac{1}{s}\sum_{i=1}^d(\Ub^\ts\xb)_{i}^2\right]\nonumber\\
=~&\sum_{j=1}^n\left[\frac{1}{s}\frac{ x_{j}^2 \|(\mathbf{U}^\ts)_{*j}\|_2^2 }{\pi_{j}}-\frac{1}{s}\|\Ub^\ts\xb\|_2^2\right]
\le~\sum_{j=1}^n\frac{\|\mathbf{U}_{j*}\|_2^2\cdot x_{j}^2 }{s\,\pi_{j}}\nonumber\,,
\end{flalign}
where the last inequality follows from the fact that $\frac{1}{s}\|\Ub^\ts\xb\|_2^2\ge 0$. This concludes the proof.

\hfill\(\square\)

\vspace{2mm}


\section{Additional Experiments}

\begin{figure*}[t]
\centering
\subfigure{%
  \includegraphics[width=0.32\textwidth]{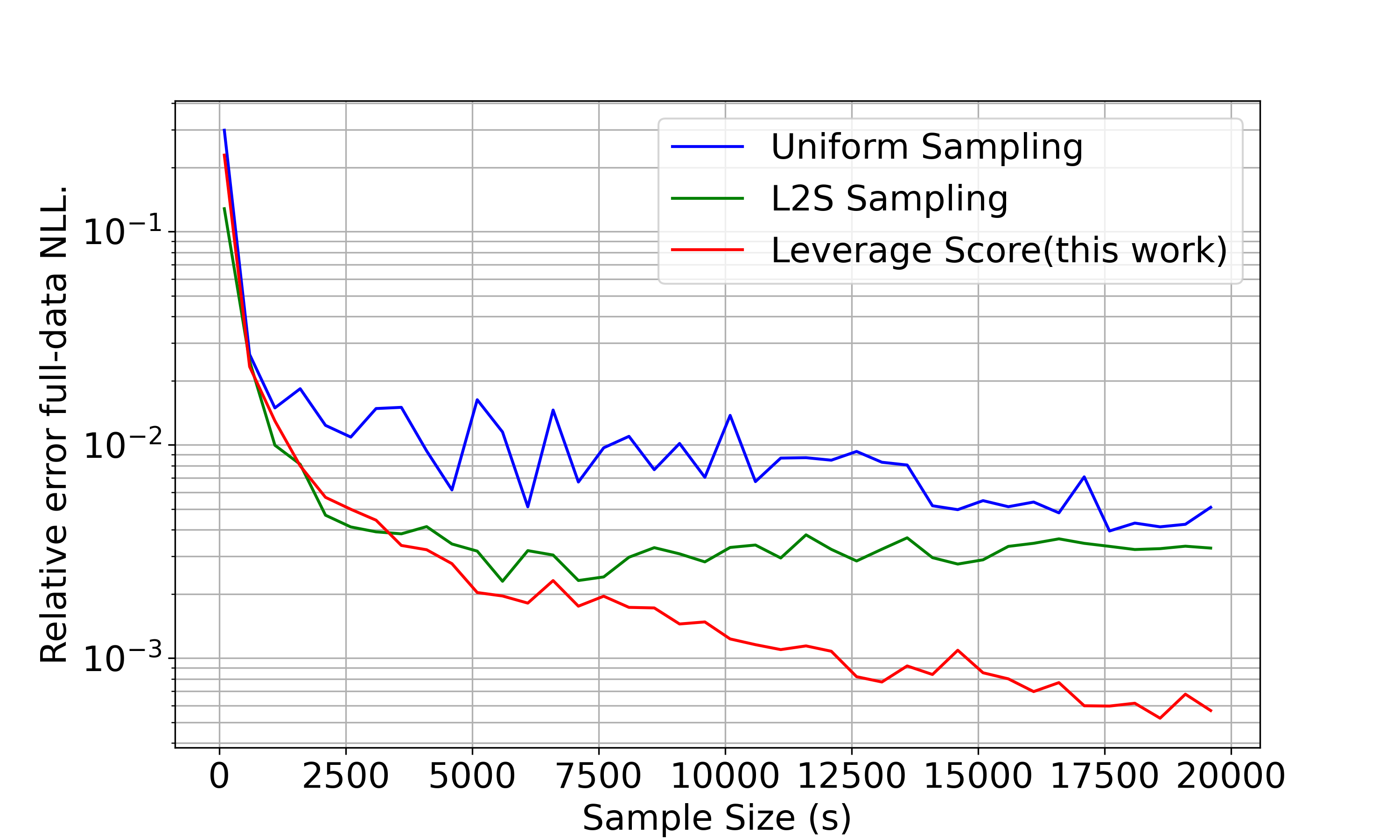}
  \label{fig:subfig19}
}
\hfill
\subfigure{%
  \includegraphics[width=0.32\textwidth]{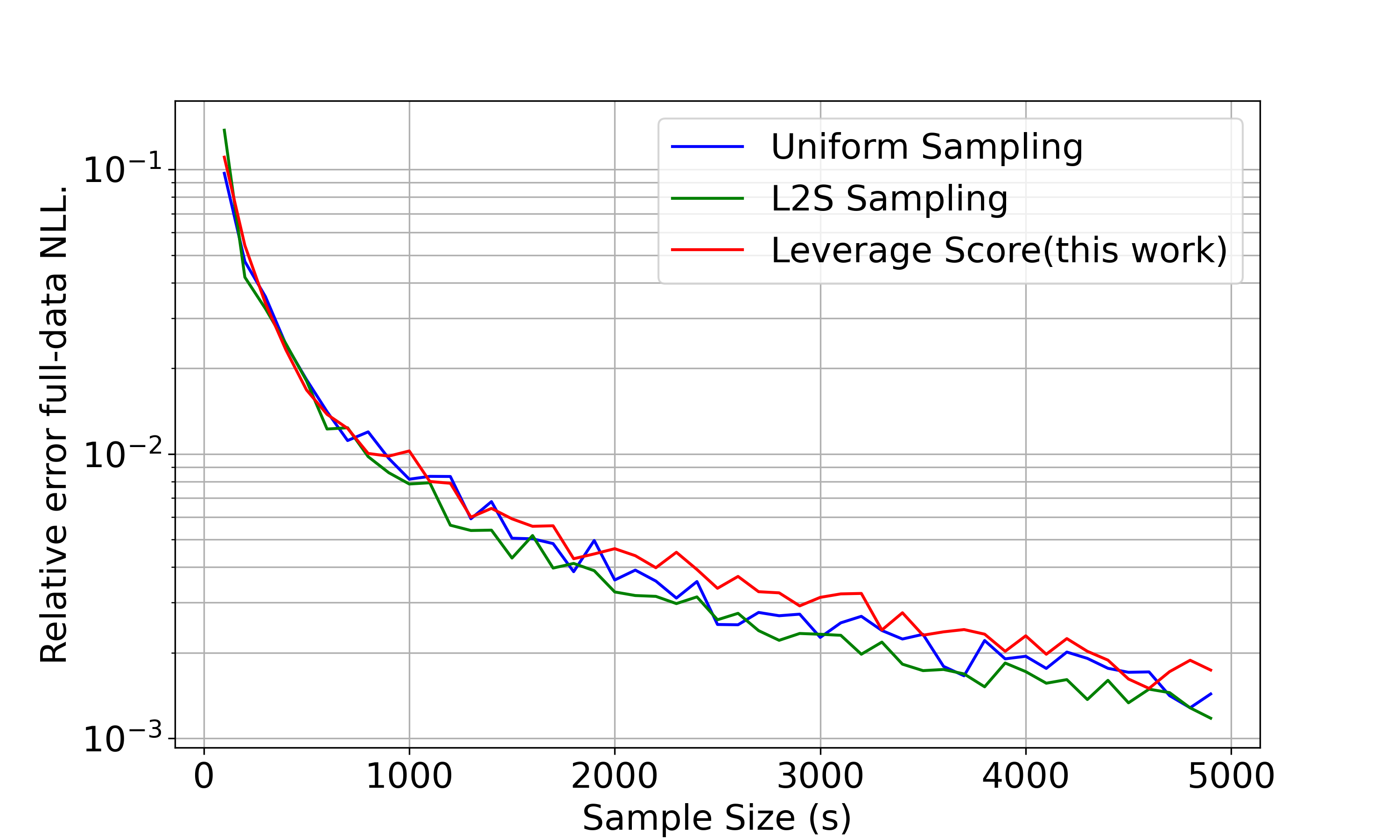}
  \label{fig:subfig20}
}
\hfill
\subfigure{%
  \includegraphics[width=0.32\textwidth]{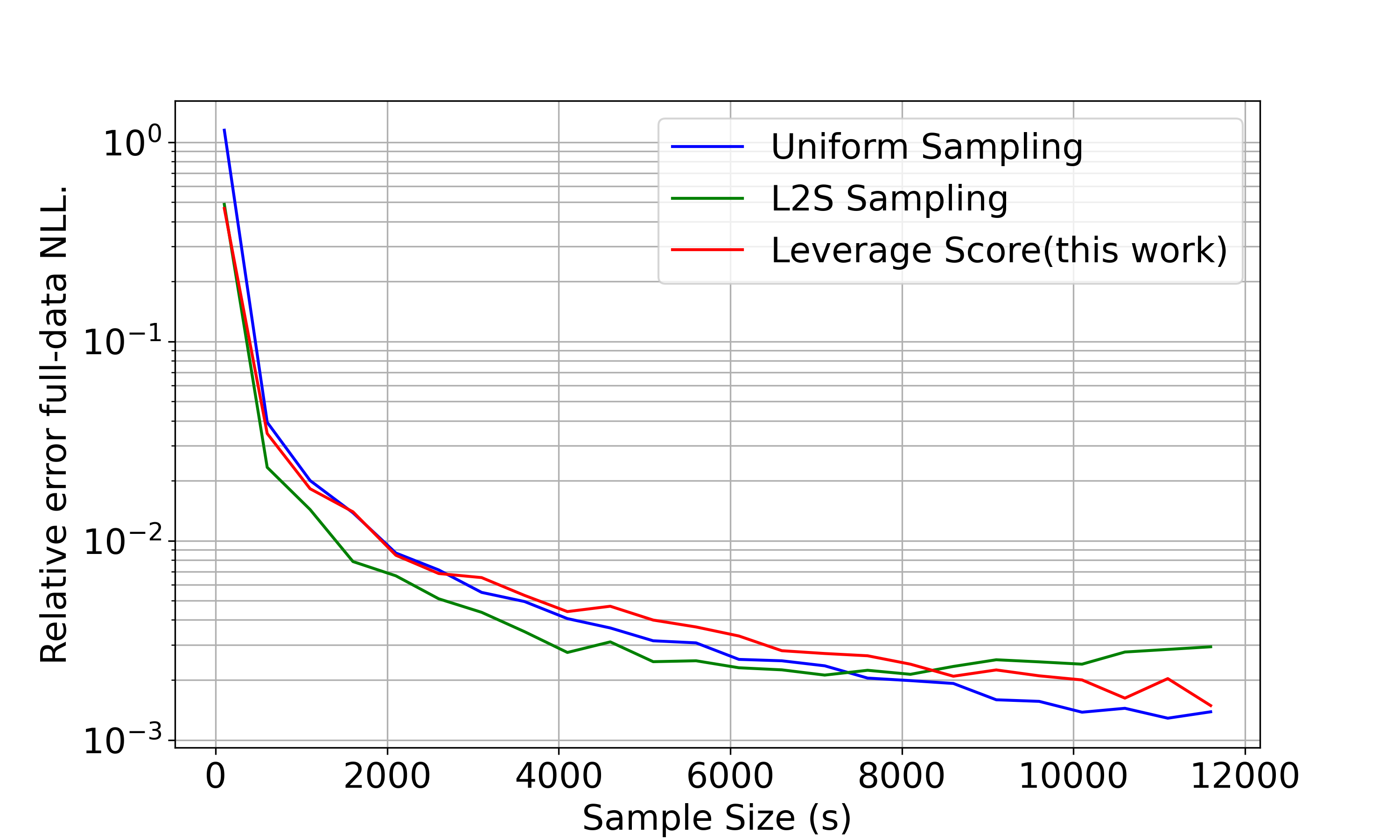}
  \label{fig:subfig21}
}

\vspace{-2mm}
\addtocounter{subfigure}{-3}
\subfigure[{\small cardio}]{%
  \includegraphics[width=0.32\textwidth]{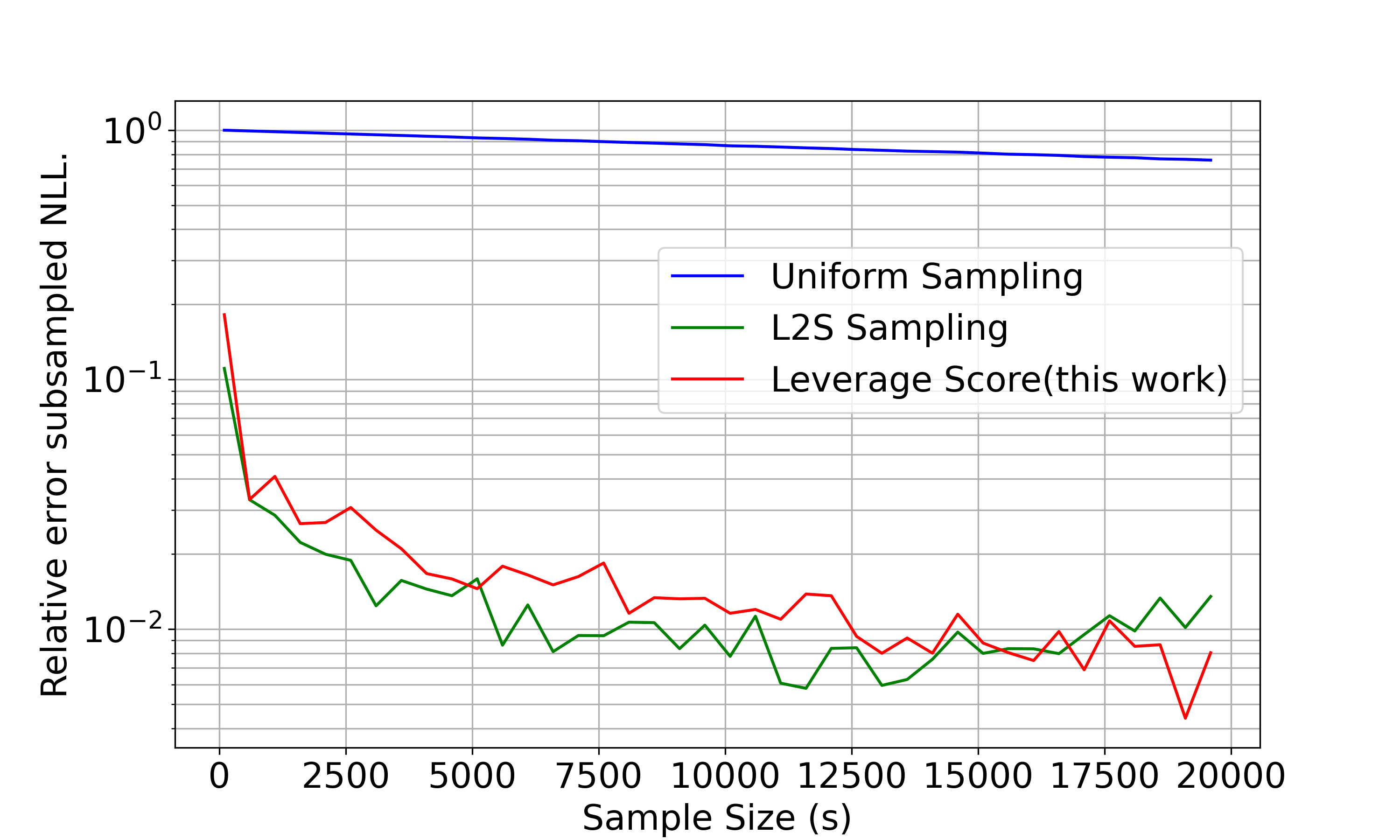}
  \label{fig:subfig16}
}
\hfill
\subfigure[{\small churn}]{%
  \includegraphics[width=0.32\textwidth]{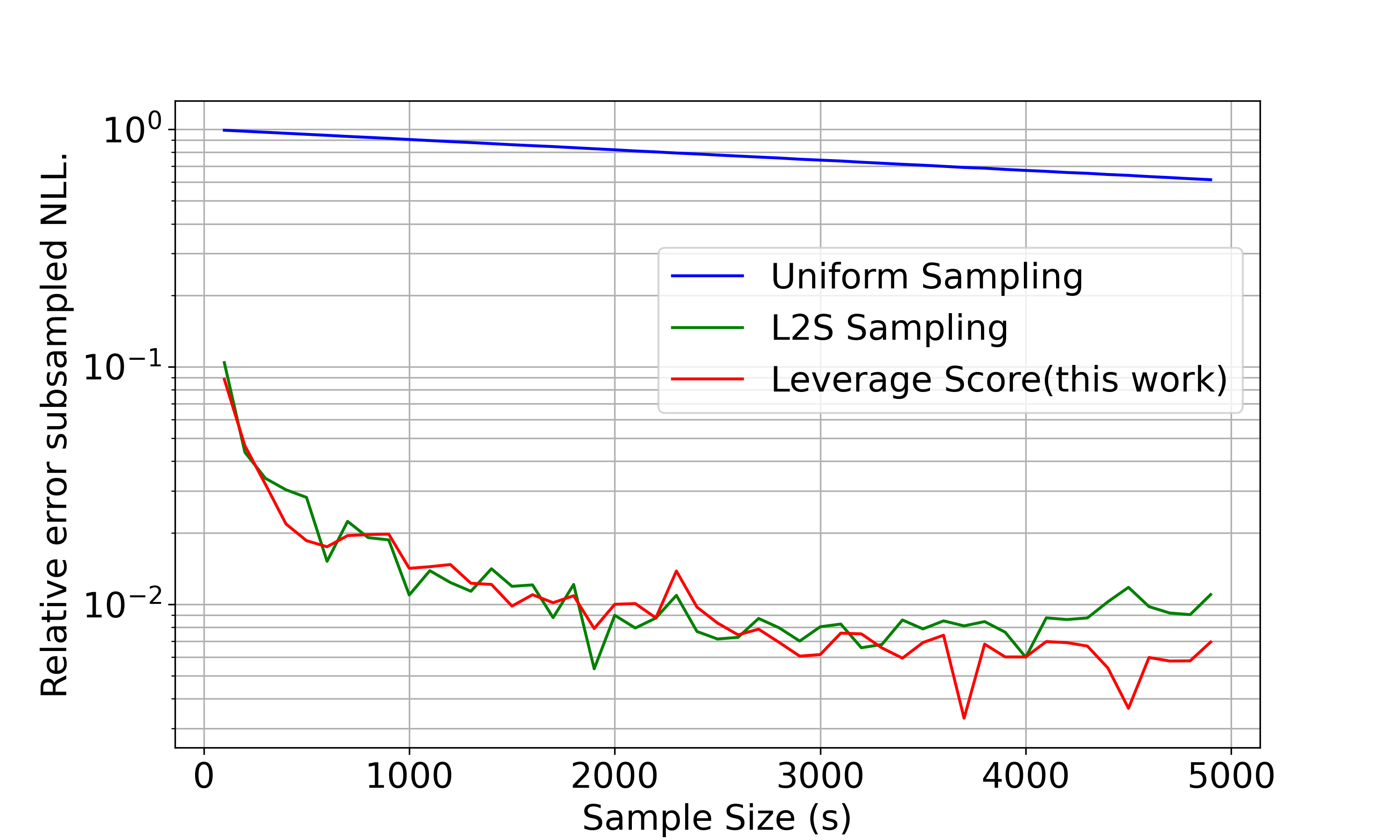}
  \label{fig:subfig17}
}
\hfill
\subfigure[{\small default}]{%
  \includegraphics[width=0.31\textwidth]{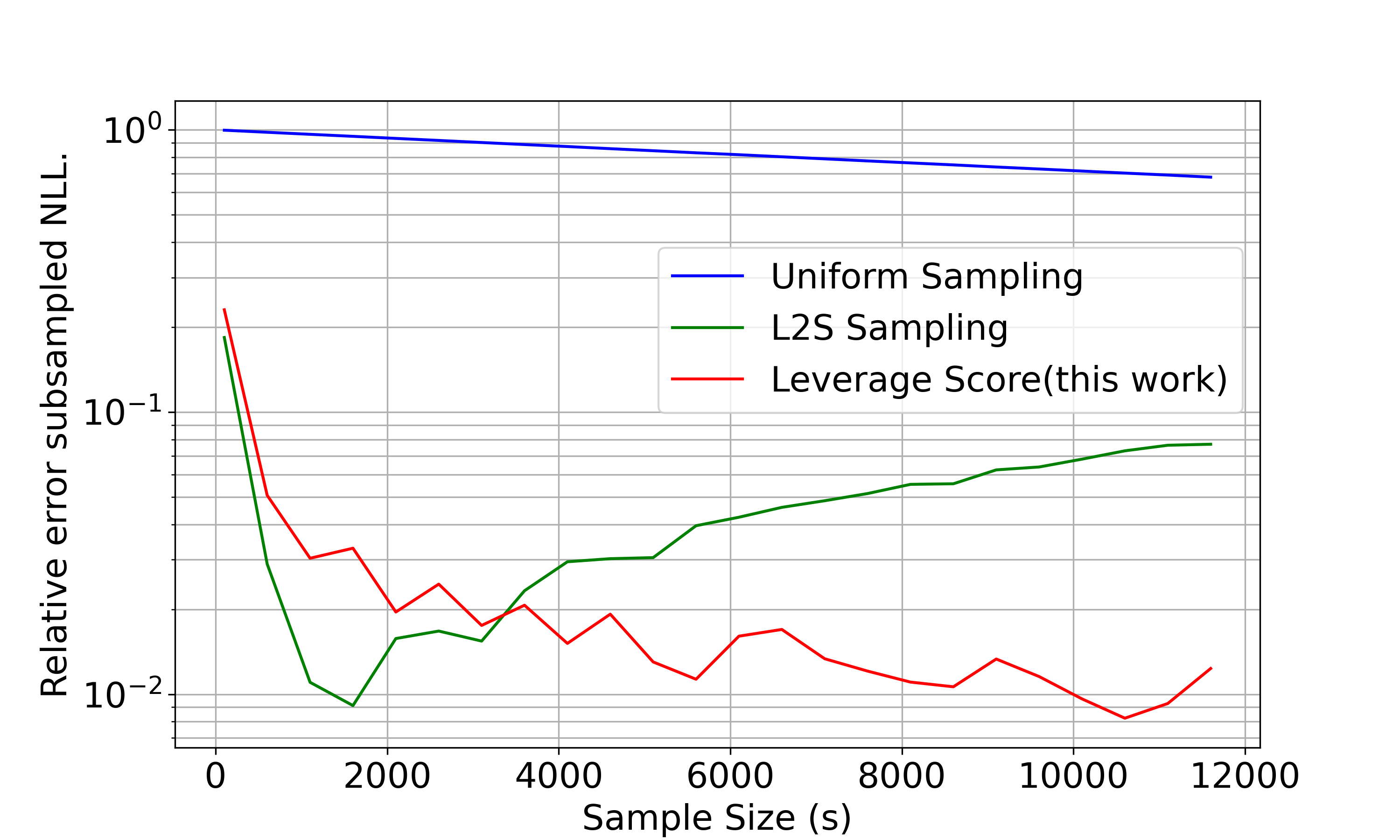}
  \label{fig:subfig18}
}
\caption{Relative error full-data negative log-likelihood (top row) and subsampled negative log-likelihood (bottom row) for all three datasets. Errors are in log-scale.
}
\label{fig:mainfig3}
\end{figure*}

We also compared our method in terms of (\emph{i}) $\nicefrac{|\ell(\hat{\betab})-\ell(\betab^*)|}{-\ell(\betab^*)}$, the relative error of the full-data negative log likelihood (Figure~\ref{fig:mainfig3} top row), which is a more common metric in other recent works and (\emph{ii}) $\nicefrac{|\Bar{\ell}(\hat{\betab})-\ell\left(\betab^*\right)|}{-\ell(\betab^*)}$, the relative error of the subsampled negative log-likelihood with respect to the full data negative log likelihood (see Figure~\ref{fig:mainfig3} bottom row). For $(ii)$, note that the first term on the numerator is $\Bar{\ell}(\hat{\betab})$ \ie, eqn.~\eqref{eq:subLL} evaluated at the output of Algorithm~\ref{algo:main} when the sampling matrix $\Sb$ is constructed using uniform sampling, leverage score sampling, and L2S method of \cite{munteanu2018coresets}.

In the top row of Figure~\ref{fig:mainfig3}, we present the relative error of the full data negative log-likelihoods in Figure~\ref{fig:mainfig3}. The trends closely resemble those shown in the first row of Figure~\ref{fig:mainfig2} (in terms of relative error of estimated probabilities). While the error due to leverage score sampling for the \emph{cardiovascular disease} dataset outperforms the other two sampling strategies, in the other two datasets (namely, \emph{Bank customer churn prediction} and \emph{Default of credit card clients}), the performance of all three sampling schemes is quite close (with L2S slightly better), and the errors decrease with the increase in $s$.
In the bottom row of Figure~\ref{fig:mainfig3}, we evaluate the relative error of the subsampled negative log-likelihood. Across all three datasets, both leverage score sampling and L2S exhibit significantly lower errors compared to uniform sampling. In the last column, our approach gets much better performance as $s$ increases. Similarly, for the other two datasets, our method outperforms L2S as $s$ gets larger.

In Figure~\ref{fig:mainfig4}, we plot the standard deviations from the $20$ runs for each of the experiments conducted.

\begin{figure*}[t]
\centering

\subfigure{%
  \includegraphics[width=0.32\textwidth]{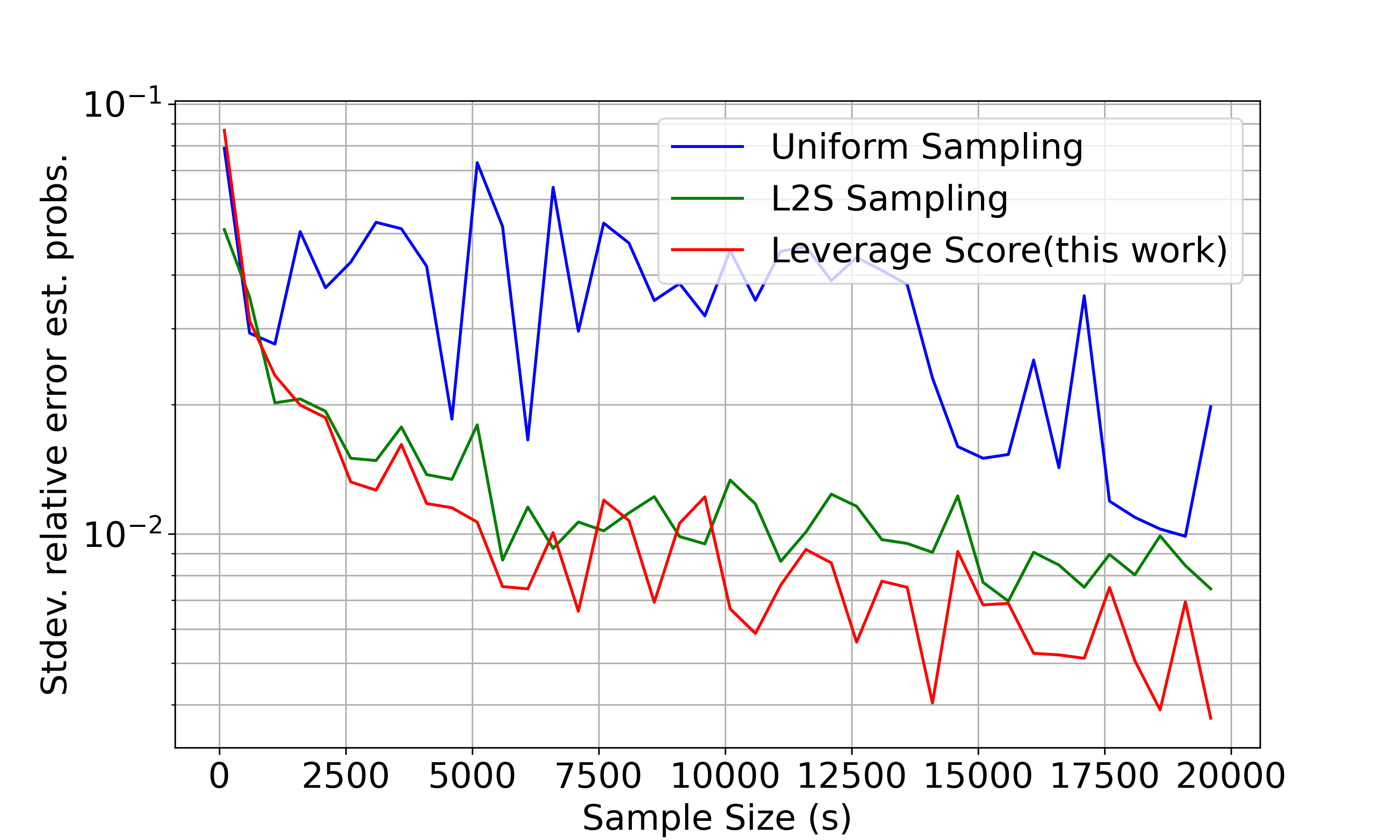}
  \label{fig:subfig22}
}
\hfill
\subfigure{%
  \includegraphics[width=0.32\textwidth]{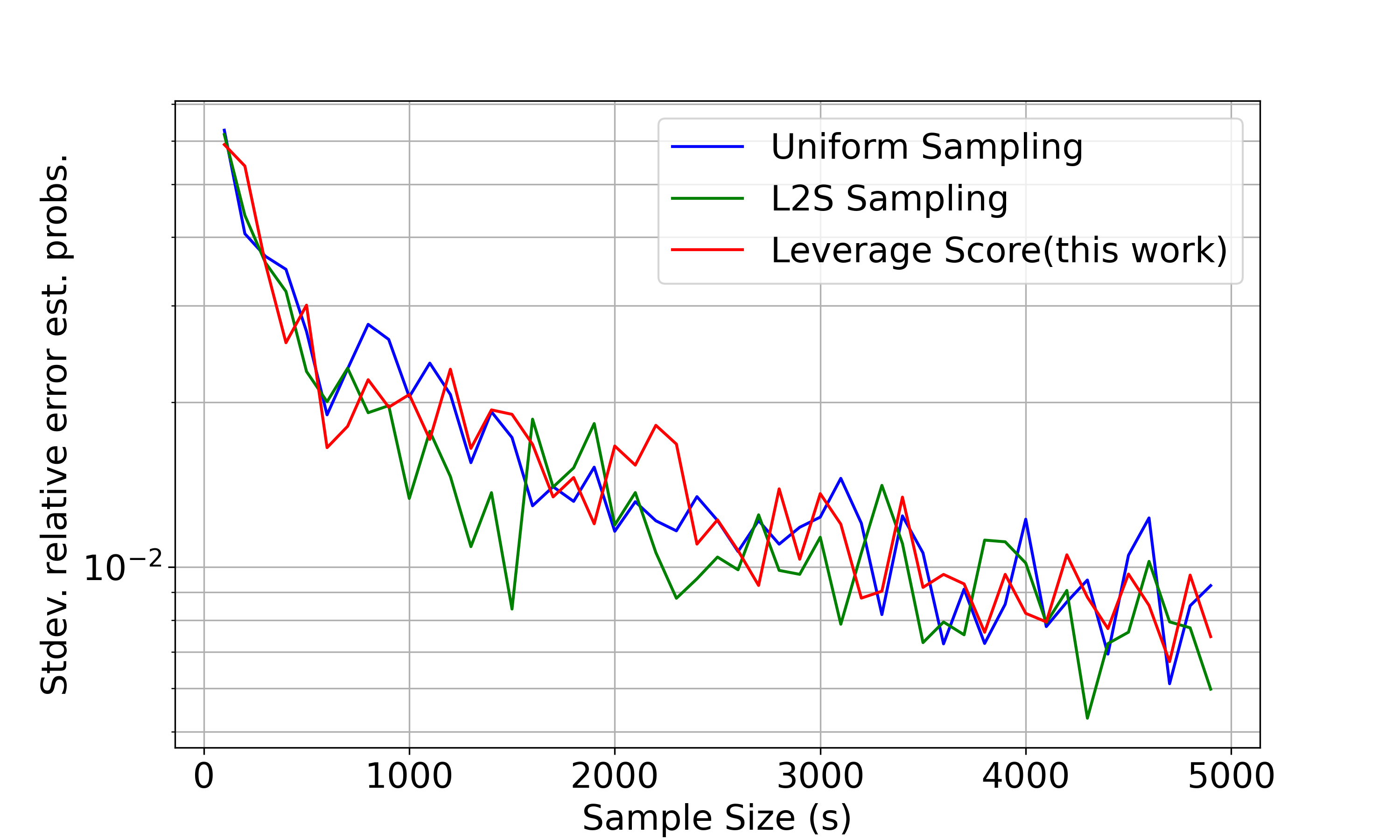}
  \label{fig:subfig23}
}
\hfill
\subfigure{
  \includegraphics[width=0.31\textwidth]{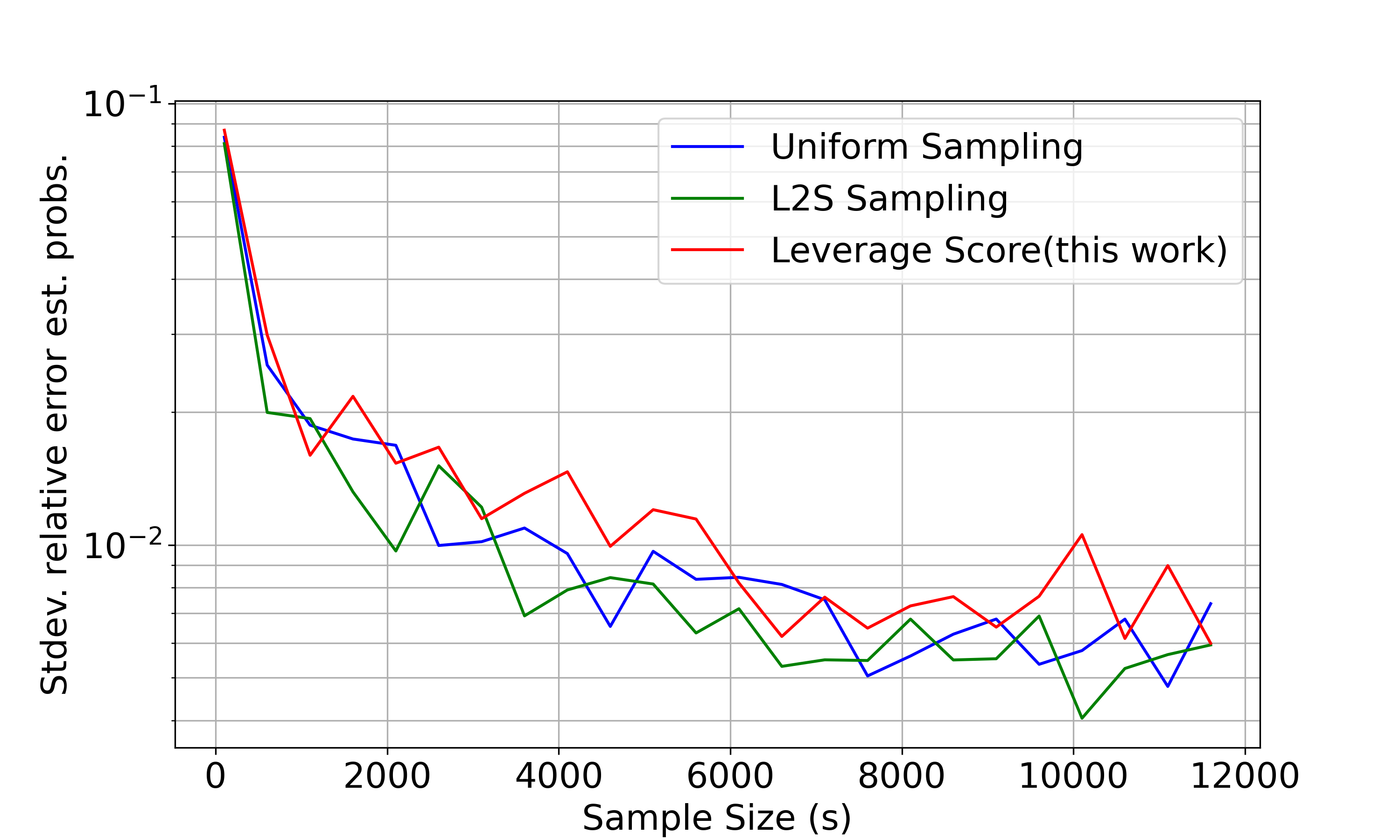}
  \label{fig:subfig24}
}

\vspace{-2mm}
\subfigure{%
  \includegraphics[width=0.32\textwidth]{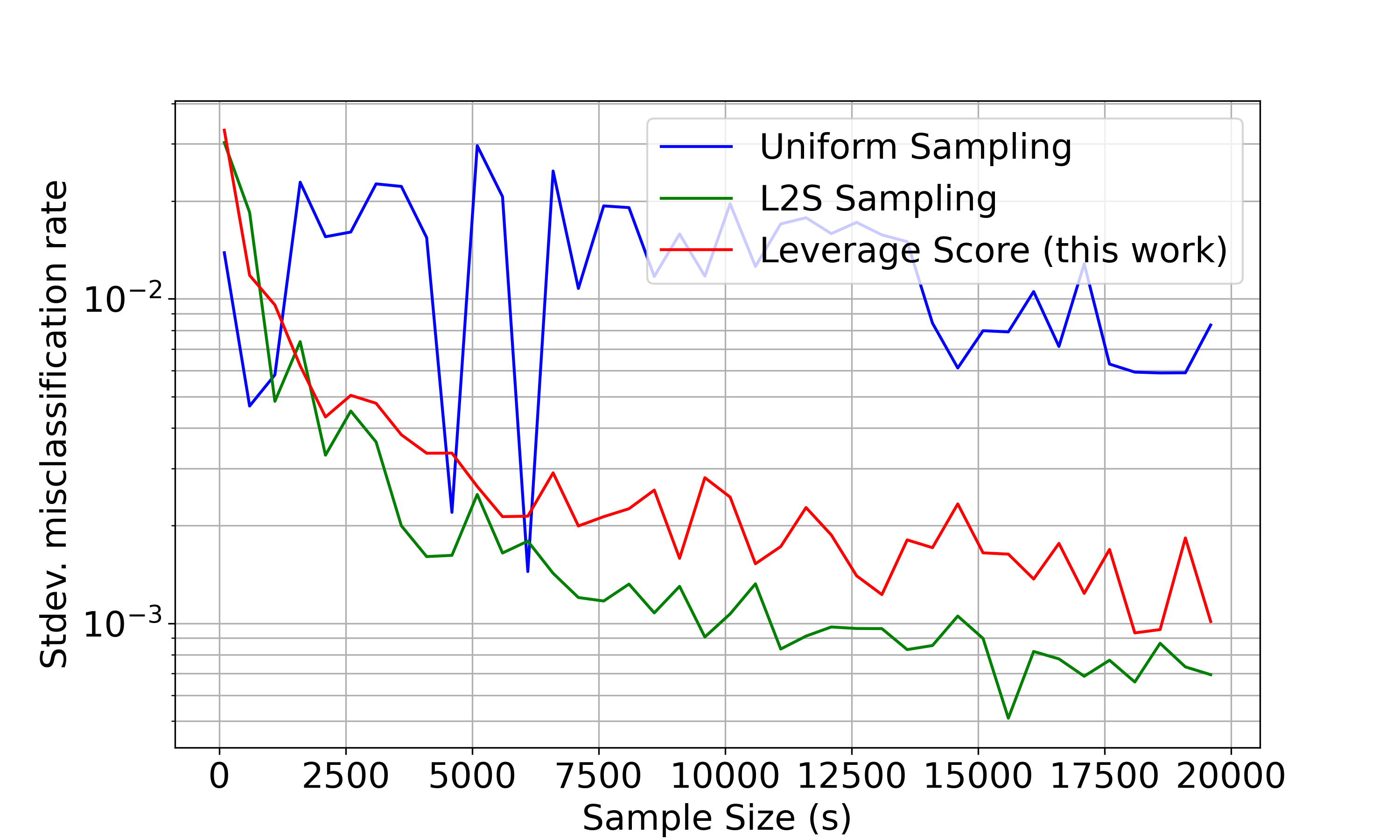}
  \label{fig:subfig25}
}
\hfill
\subfigure{%
  \includegraphics[width=0.32\textwidth]{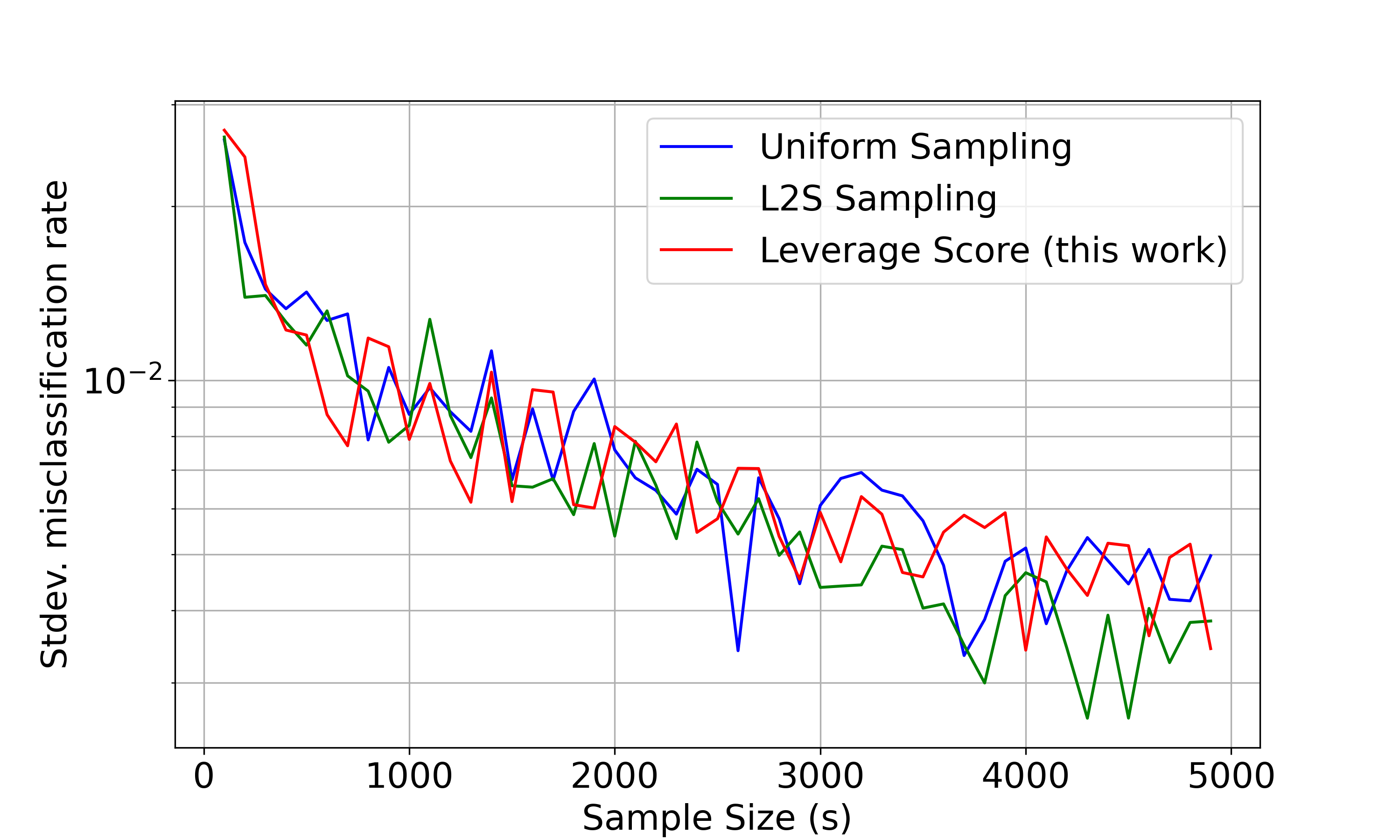}
  \label{fig:subfig26}
}
\hfill
\subfigure{%
  \includegraphics[width=0.32\textwidth]{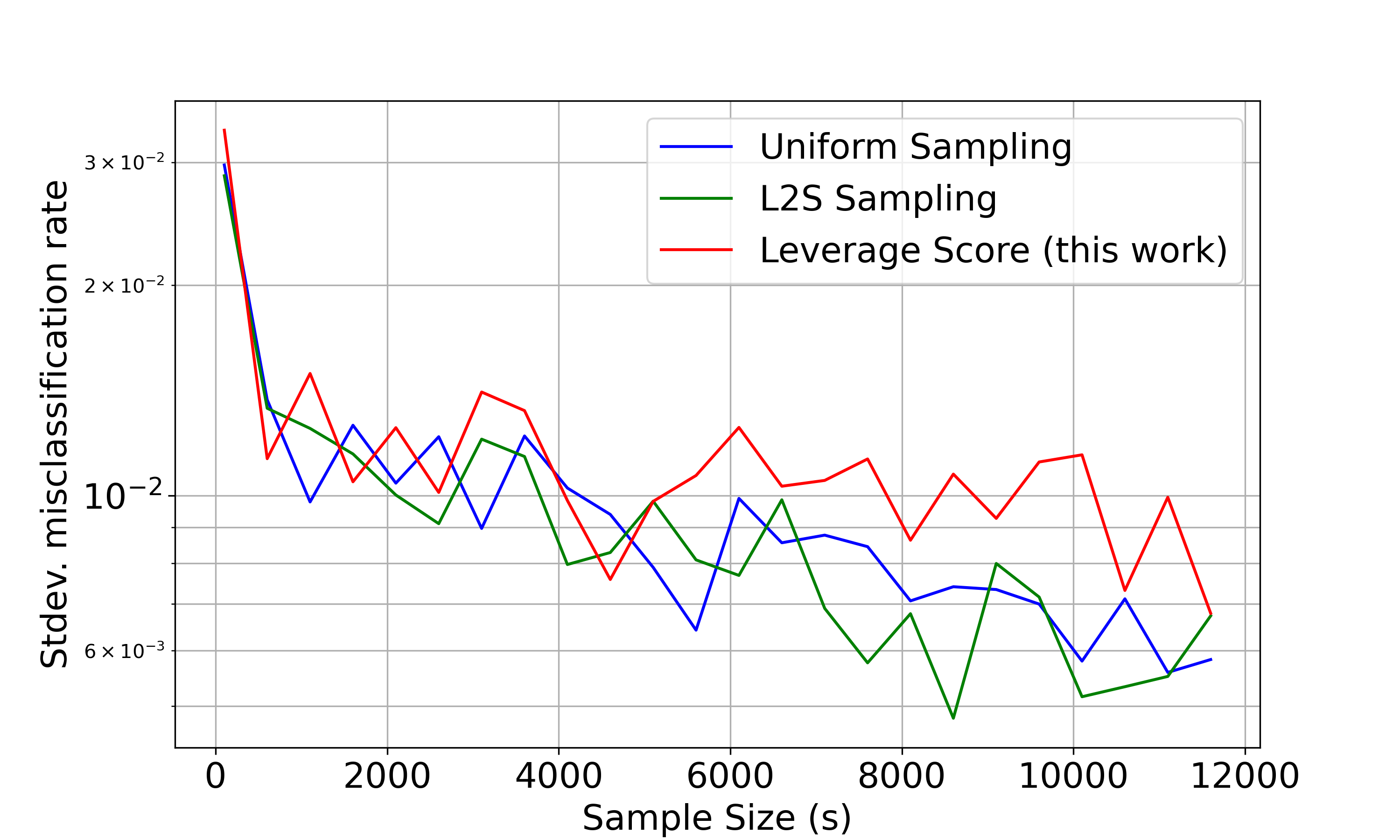}
  \label{fig:subfig27}
}

\vspace{-2mm}

\subfigure{%
  \includegraphics[width=0.32\textwidth]{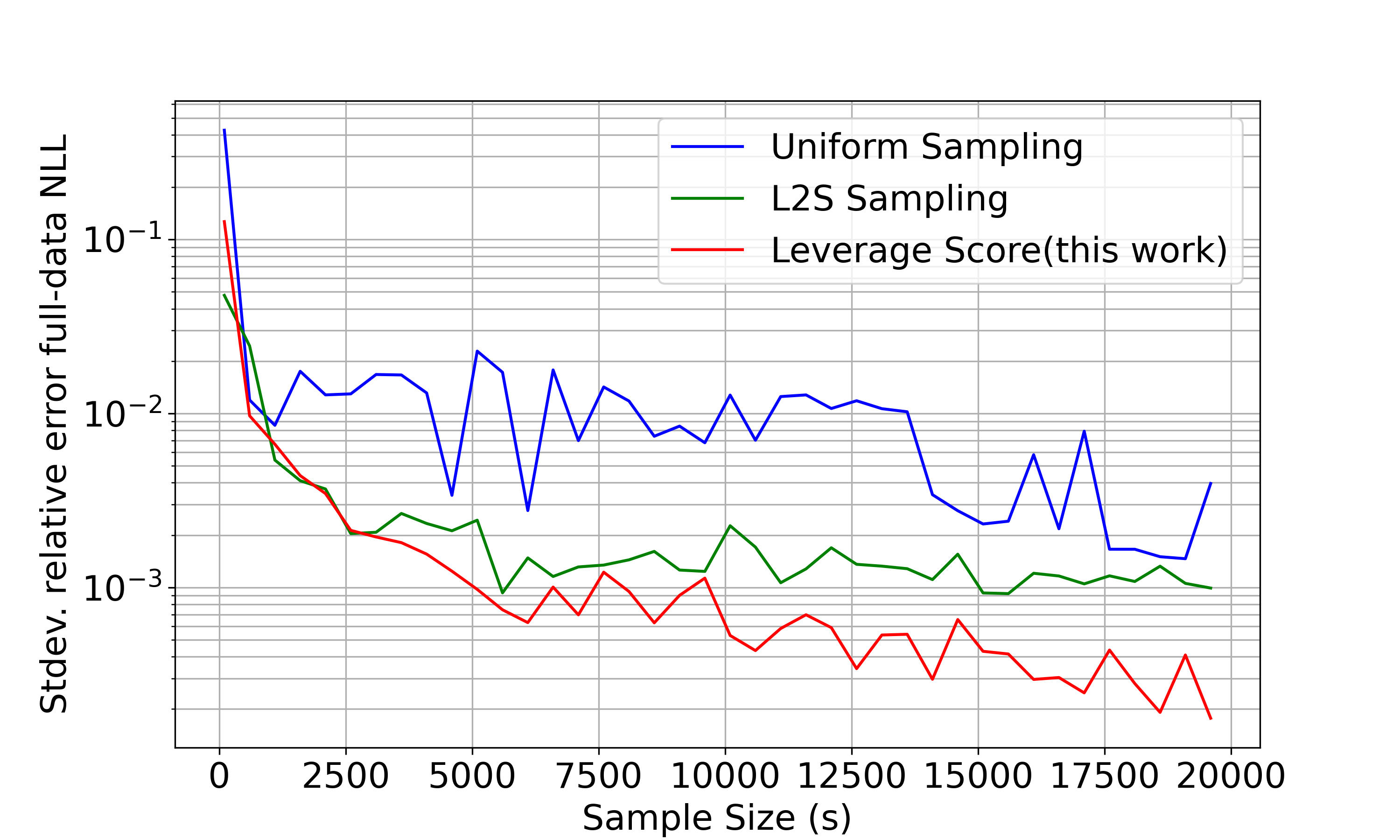}
  \label{fig:subfig31}
}
\hfill
\subfigure{%
  \includegraphics[width=0.32\textwidth]{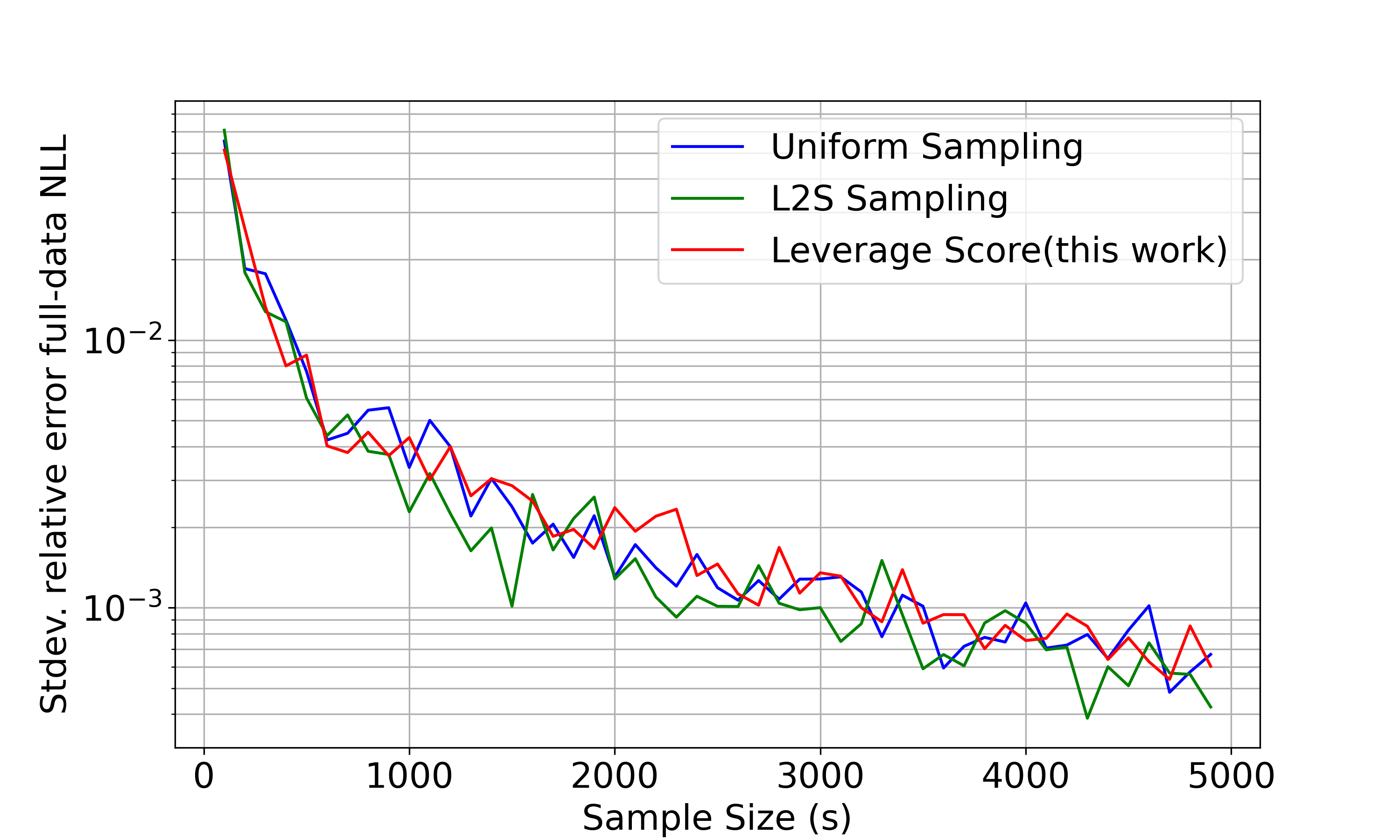}
  \label{fig:subfig32}
}
\hfill
\subfigure{%
  \includegraphics[width=0.32\textwidth]{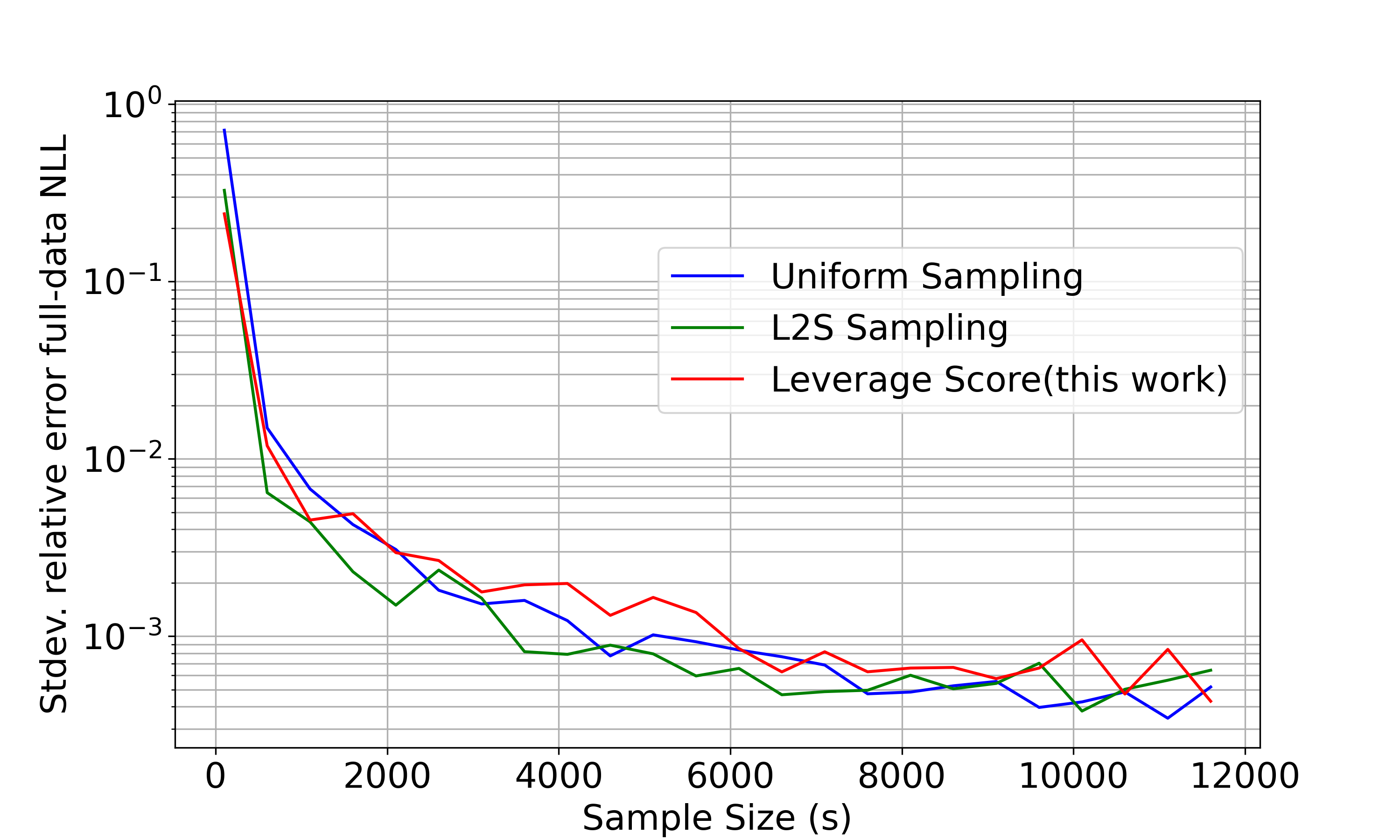}
  \label{fig:subfig33}
}

\addtocounter{subfigure}{-9}

\subfigure[{\small cardio}]{%
  \includegraphics[width=0.32\textwidth]{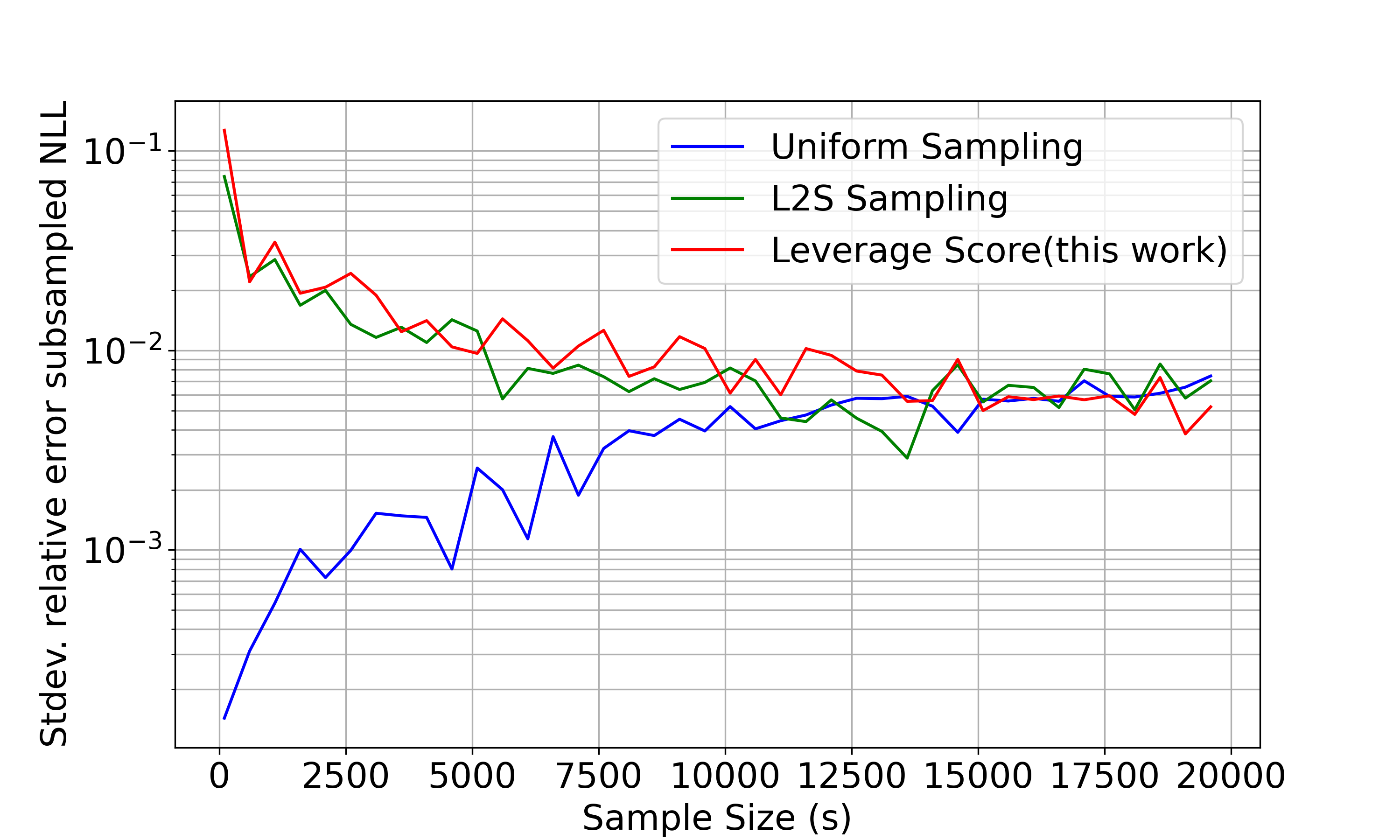}
  \label{fig:subfig28}
}
\hfill
\subfigure[{\small churn}]{%
  \includegraphics[width=0.32\textwidth]{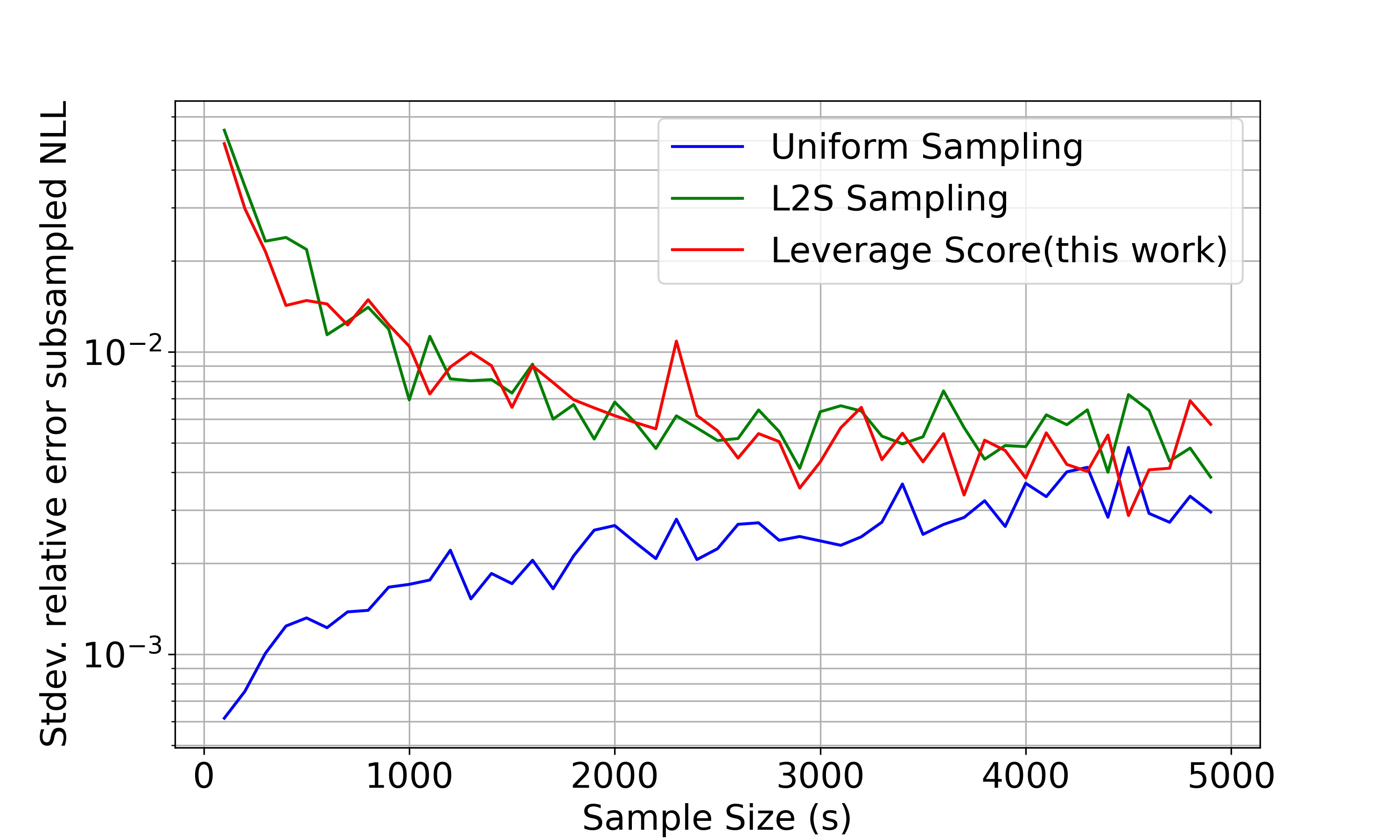}
  \label{fig:subfig29}
}
\hfill
\subfigure[{\small default}]{%
  \includegraphics[width=0.32\textwidth]{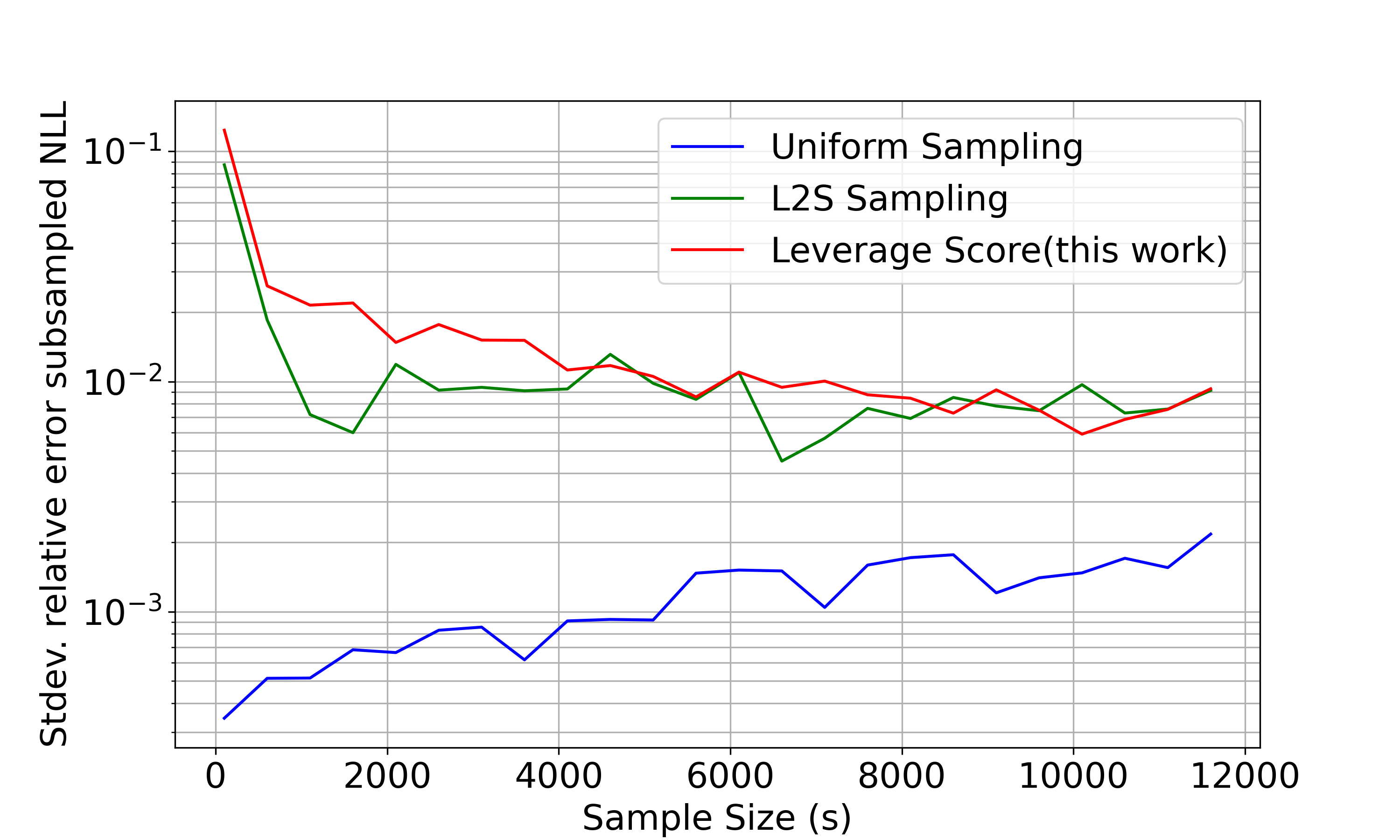}
  \label{fig:subfig30}
}

\caption{
Standard deviations for all the metrics in Figures~\ref{fig:mainfig2} and \ref{fig:mainfig3}.
All the errors are in log-scale.
}
\label{fig:mainfig4}
\end{figure*}

\end{document}